\documentclass[11pt]{article}

\usepackage[letterpaper,margin=1in]{geometry}
\usepackage[T1]{fontenc}
\usepackage{lmodern}
\usepackage{microtype}
\usepackage{amsmath,amssymb,amsthm,mathtools,bm}
\usepackage{booktabs,array,multirow,enumitem}
\usepackage{graphicx}
\usepackage{xcolor}
\usepackage[numbers,sort&compress]{natbib}
\usepackage{hyperref}
\usepackage{url}
\usepackage{tikz}
\usetikzlibrary{positioning,arrows.meta}
\usepackage{caption}
\usepackage{longtable}
\usepackage{float}

\graphicspath{{figures/}}
\setlist{nosep}
\allowdisplaybreaks
\hypersetup{
  hidelinks,
  pdftitle={Intrinsic Interaction Geometry Controls the Low-Rank Complexity of Softmax Attention},
  pdfauthor={Yuhe Sui; Jianing Zhang; Yingzhi Tang},
  pdfsubject={ORCID metadata: Yuhe Sui 0009-0002-7456-0804; Yingzhi Tang 0009-0004-7571-0100},
  pdfkeywords={softmax attention, approximation rank, intrinsic geometry, gauge invariance, ORCID 0009-0002-7456-0804, ORCID 0009-0004-7571-0100}
}

\newtheorem{theorem}{Theorem}
\newtheorem{proposition}[theorem]{Proposition}
\newtheorem{lemma}[theorem]{Lemma}
\newtheorem{corollary}[theorem]{Corollary}
\theoremstyle{definition}

\theoremstyle{remark}
\newtheorem{remark}[theorem]{Remark}

\newcommand{\rowone}{\mathrm{row},1}
\newcommand{\rank}{\operatorname{rank}}
\newcommand{\rad}{\operatorname{rad}}
\newcommand{\osc}{\operatorname{osc}}
\newcommand{\softmax}{\operatorname{softmax}}
\newcommand{\R}{\mathbb{R}}
\newcommand{\one}{\mathbf{1}}
\newcommand{\eps}{\varepsilon}
\newcommand{\kap}{\kappa}
\newcommand{\ta}{\tau}

\newcommand{\Cspace}{\mathcal C}
\newcommand{\Lspace}{\mathcal L}
\newcommand{\gammatwo}{\gamma_2}

\title{Intrinsic Interaction Geometry Controls the Low-Rank Complexity of Softmax Attention}
\author{%
Yuhe Sui$^{1,2}$ \quad Jianing Zhang$^{1,3}$ \quad Yingzhi Tang$^{1,4}$\\[5pt]
\small $^1$Quantitative Research Society, Singapore\\
\small $^2$Nanyang Technological University, Singapore\\
\small $^3$Carnegie Mellon University, Pittsburgh, PA, USA\\
\small $^4$The Chinese University of Hong Kong, Shenzhen, China
}
\date{}

\begin{document}
\maketitle
\footnotetext{Large language models were used as general-purpose assistive tools for drafting, coding support, mathematical checking, and language editing. The authors independently reviewed the theorem statements, computations, empirical claims, and final manuscript and take full responsibility for the content.}

\begin{abstract}
How much matrix rank is required to preserve every bounded value output of normalized softmax attention? We study the unrestricted maximum-row-$\ell_1$ approximation rank $r_\eps(A)$, exactly the least rank achieving uniform error over all bounded vector-valued values. Row softmax exposes the intrinsic interaction $C=P_m(\log A)P_N$, whereas invertible $Q/K$ gauges leave $A$ fixed while changing the Euclidean geometry of a chosen query/key factorization. We replace that coordinate-dependent description by a projective residual $q(C-T)$ and an attained factor-radius size $\kap(T)$. For every rank-$r$ retained interaction with $\ta(T)<\eps$, we prove
\[
r_\eps(A)\le \min\left\{N,C_r\left(1+\frac{\kap(T)}{(\eps-\ta(T))^2}\right)^{r/2}\right\},
\]
with the same unknown dimension constant as the underlying weighted Gibbs-row cover. The profile is gauge invariant, termwise no worse than native retained-subspace bounds at the same declared dimension, and has a worst-case sharp $r/2$ size exponent at fixed $r$ and $\eps$. We then measure $r_\eps(A)$ directly on learned attention using 9,978 certified brackets across BERT, GPT-2, Qwen2.5, and two ViT checkpoints; where certificates do not close, the optimum remains interval-valued. A pre-specified 2,302-cell held-out study further shows that the historical native-coordinate geometry block contains coarse, mostly head-level information but no detectable incremental information beyond a strong calibrated baseline. The new intrinsic descriptor is not evaluated in that study. Together, the theory and measurements distinguish an operator-intrinsic complexity control from a stronger empirical explanation that the learned-head evidence does not support.
\end{abstract}

\section{Introduction}
Softmax attention is a normalized matrix operator. Its output is not the raw logit kernel but a row-stochastic matrix $A$ acting on values. The relevant representation-complexity question is therefore not whether the logits or one particular query/key factorization have low algebraic rank, but how much rank is needed to reproduce the action of $A$ on every bounded value assignment.

The usual query/key coordinates are not unique. For any invertible $S$,
\[
QK^\top=(QS)(KS^{-\top})^\top,
\]
so the logits and normalized attention are unchanged even though Euclidean radii and SVD-selected directions can move. A theorem expressed through one chosen pair $(Q,K)$ can thus remain mathematically valid while its geometric parameters fail to be properties of the normalized operator.

We resolve this tension by moving from a factorization to the visible interaction itself. Row softmax removes row constants, while common key-only terms act as positive base weights. The canonical interaction that survives both operations is
\[
C=P_m(\log A)P_N.
\]
For a low-rank retained interaction $T$, we measure the residual by a projective row-oscillation norm $q(C-T)$ and the retained size by the smallest product of Chebyshev radii over all factorizations of $T$. The latter quantity, $\kap(T)$, attains in rank dimension and is a quotient of the Hilbert factorization norm $\gamma_2$. These two quantities must come from the same $T$. Combined with the sharp softmax perturbation inequality and the weighted Gibbs-row cover, they give a gauge-invariant robust upper profile for the actual output-preserving rank.

The theorem does not imply that a convenient native diagnostic is an empirically unique explanation of trained heads. We therefore measure $r_\eps(A)$ directly with certified lower and upper witnesses, then ask whether the historical native-coordinate SoftRank feature block explains constructive rank demand beyond calibrated logit, concentration, and weight-space observables. On the pre-specified panel it does not. This falsification separates theorem-level control from empirical explanatory uniqueness.

Our contributions are:
\begin{enumerate}[leftmargin=1.6em]
\item \textbf{Operator-intrinsic complexity theory.} Maximum-row-$\ell_1$ error is exactly uniform bounded-value output error; the exact visible interaction is $C=P_m(\log A)P_N$; and $q$ together with the attained factor-radius size $\kap$ gives a gauge-invariant upper profile.
\item \textbf{Sharpness and necessity.} A finite-grid family gives $r_\eps(A)=\Theta_{r,\eps}(\kap(C)^{r/2})$ at fixed $r,\eps$, while rank-one obstructions show that interaction rank and residual alone cannot control the target independently of size.
\item \textbf{Direct learned-head measurement.} Certified lower and constructive upper witnesses bracket $r_\eps(A)$ across BERT, GPT-2, Qwen2.5, and two ViT checkpoints, with exact small-regime controls validating the target semantics.
\item \textbf{A strong-baseline falsification.} The historical native geometry block contains coarse head-level information, but contributes no detectable held-out information beyond the calibrated B1 baseline on the pre-specified panel. This result concerns the learned factorization; the new intrinsic descriptor is not evaluated empirically.
\end{enumerate}

\begin{figure}[t]
\centering
\begin{tikzpicture}[
  box/.style={draw=black!60, rounded corners=2pt, minimum width=0.280\linewidth,
              minimum height=1.15cm, align=center, inner sep=4pt, font=\small},
  theory/.style={box, fill=blue!3}, evidence/.style={box, fill=green!3},
  flow/.style={-{Latex[length=2mm]}, line width=0.8pt, draw=black!65}]
\node[theory] (n1) {\textbf{1. Output-preserving target}\\[-1pt] max-row $\ell_1$; $r_\eps(A)$};
\node[theory, right=5mm of n1] (n2) {\textbf{2. Visible quotient}\\[-1pt] $C=P_m(\log A)P_N$; $R=\rank(C)$};
\node[theory, right=5mm of n2] (n3) {\textbf{3. Gauge boundary}\\[-1pt] $Q\mapsto QS$, $K\mapsto KS^{-\top}$\\native $\Gamma$ moves};
\node[evidence, below=8mm of n1] (n4) {\textbf{4. Intrinsic profile}\\[-1pt] $q,\ta,\kap$; certified upper law};
\node[evidence, right=5mm of n4] (n5) {\textbf{5. Direct measurement}\\[-1pt] certified $[L,U]$ on learned heads};
\node[evidence, right=5mm of n5] (n6) {\textbf{6. Falsification}\\[-1pt] $B1$ vs. $B1{+}G$; held-out null};
\draw[flow] (n1.east) -- (n2.west); \draw[flow] (n2.east) -- (n3.west);
\draw[flow] (n3.south) -- ++(0,-3mm) -| (n4.north);
\draw[flow] (n4.east) -- (n5.west); \draw[flow] (n5.east) -- (n6.west);
\end{tikzpicture}
\caption{SoftRank's argument in one flow. The appendices give complete proofs for the weighted Gibbs cover, the quotient and residual theory, the factorization geometry, gauge boundaries, sharpness constructions, and finite certificates.}
\label{fig:story}
\end{figure}
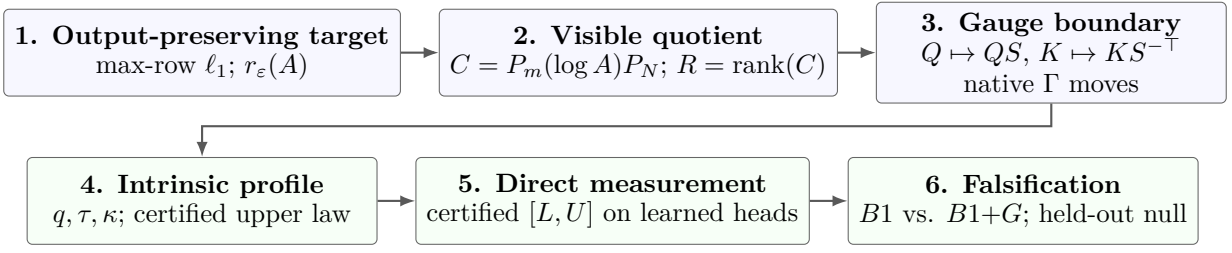

\section{The output-preserving target and the visible quotient}
\label{sec:target}
Let $A\in\R^{m\times N}$ be row-stochastic. For a matrix $M$ define
\[
\|M\|_{\rowone}=\max_i\sum_j|M_{ij}|=\|M\|_{\infty\to\infty}.
\]
For $0<\eps<1$,
\begin{equation}
r_\eps(A)=\min\{\rank(B):B\in\R^{m\times N},\ \|A-B\|_{\rowone}\le\eps\}.
\label{eq:target}
\end{equation}
The competitor is unrestricted: it need not be stochastic, positive, sparse, or itself an attention matrix. If $E$ is any nonzero normed value space, then
\begin{equation}
\|M\|_{\rowone}=
\sup_{\max_j\|v_j\|_E\le1}\max_i\left\|\sum_jM_{ij}v_j\right\|_E.
\label{eq:valueidentity}
\end{equation}
Thus $r_\eps(A)$ is exactly the least unrestricted rank preserving every bounded vector-valued attention output to error $\eps$.

Suppose $A=\softmax_{\rm row}(H)$ with finite logits on one common positive support. Row softmax identifies exactly the equivalence class
\[
H\sim H+a\one^\top.
\]
Let $P_m=I_m-\one\one^\top/m$ and $P_N=I_N-\one\one^\top/N$. Because $\log A$ differs from $H$ only by row constants,
\begin{equation}
C=P_mHP_N=P_m(\log A)P_N
\label{eq:C}
\end{equation}
is determined by $A$. For nonzero logit scale, the exact visible interaction dimension is
\begin{equation}
R=\rank(C).
\label{eq:R}
\end{equation}
The common column term is intentionally retained as a positive base-weight vector rather than discarded as a softmax equivalence. Appendix~\ref{app:quotient-deep} derives this quotient in both anchored query/key coordinates and the matrix representation $C$.

\section{Native geometry is valid but representation-specific}
\label{sec:gauge}
For a chosen quotient factorization, the inherited exact bound has the form
\begin{equation}
r_\eps(A)\le\min\left\{p,C_R\left(1+\frac{\Gamma}{\eps^2}\right)^{R/2}\right\},
\qquad p=\min(m,N),\quad \Gamma=\beta\rho_Q\rho_K,
\label{eq:nativeexact}
\end{equation}
with a robust retained-subspace extension that pays a residual allowance $\tau_W$ and replaces $\Gamma$ by $\Gamma_W$. These statements remain correct for the supplied factorization.

They are not operator invariants. Under any invertible $S$,
\[
Q\mapsto QS,\qquad K\mapsto KS^{-\top},
\]
$H=\beta QK^\top$, $A$, $C$, $R$, and every direct function of $A$ remain fixed, while native radii, retained widths, and visible-key SVD choices can move. For example, with $Q=K=(0,e_1,e_2)^\top$ and $S=\mathrm{diag}(t,t^{-1})$,
\[
\Gamma_t=\frac\beta4(t^2+t^{-2})\to\infty
\]
although $A$ is unchanged. Appendix~\ref{app:gauge-deep} gives the exact transformation laws, narrow rank-one exception, and counterexamples showing that even optimization over all native orthogonal retained subspaces does not restore invariance.

\section{An intrinsic robust interaction profile}
\label{sec:intrinsic}
For $E\in\R^{m\times N}$ write
\begin{equation}
q(E)=\min_{b\in\R^N}\max_i\osc(E_{i\cdot}-b^\top),
\qquad \osc(x)=\max_jx_j-\min_jx_j.
\label{eq:q}
\end{equation}
The sharp softmax perturbation inequality is
\begin{equation}
\|\softmax(x+h)-\softmax(x)\|_1
\le2\tanh\!\left(\frac{\osc(h)}4\right),
\label{eq:perturb}
\end{equation}
so for a centred retained interaction $T$ define
\begin{equation}
\ta(T)=2\tanh\!\left(\frac{q(C-T)}4\right).
\end{equation}
The intrinsic size of $T$ is
\begin{equation}
\kap(T)=\inf_{T=XY^\top}\rad(X)\rad(Y),
\qquad \rad(X)=\inf_a\max_i\|X_{i\cdot}-a\|_2.
\label{eq:kappa}
\end{equation}
On doubly centred matrices this infimum attains in rank dimension and equals a projective quotient of the Hilbert factorization norm $\gamma_2$; Appendix~\ref{app:kappa-deep} proves the quotient identity, SDP representation, attainment, norm comparisons, and the failure of uncentred attainment.

\begin{theorem}[Intrinsic robust SoftRank profile]
\label{thm:intrinsic}
Let $A=\softmax_{\rm row}(H)$ have finite logits on a common positive support, let $0<\eps<1$, and put $C=P_mHP_N$. Let $T$ be centred with $\rank(T)\le r$, $T\ne0$, and $\ta(T)<\eps$. Then
\begin{equation}
\boxed{r_\eps(A)\le\min\left\{p,C_r\left(1+\frac{\kap(T)}{(\eps-\ta(T))^2}\right)^{r/2}\right\}.}
\label{eq:intrinsic}
\end{equation}
Here $C_r$ is the same unknown dimension-dependent constant as in the weighted Gibbs-row cover. If $T=0$ and $\ta(T)\le\eps$, rank one is exact.
\end{theorem}

\begin{proof}
Write $H=C+a_0\one^\top+\one b_0^\top$. Choose a column shift $b$ attaining $q(C-T)$ and put $\widetilde H=T+\one(b_0+b)^\top$. Row constants disappear from oscillation, so \eqref{eq:perturb} gives $\|A-\widetilde A\|_{\rowone}\le\ta(T)$ for $\widetilde A=\softmax_{\rm row}(\widetilde H)$. Choose an attained factorization $T=XY^\top$ with radius product $\kap(T)$ and pad it to dimension $r$ if needed. Recentring $X,Y$ at their minimum enclosing balls changes $T$ only by additive row and column terms; the former vanish under softmax and the latter join the common positive base weights. The remaining family is therefore an $r$-dimensional weighted Gibbs family with radius product $\kap(T)$. Apply the weighted Gibbs-row cover at $e=\eps-\ta(T)>0$, then use the triangle inequality. The cap $p$ is supplied by $B=A$. The zero branch has identical rows after the residual adjustment and hence rank one.
\end{proof}

The theorem contains no explicit $\beta$ because $C$ and $T$ are built from the already-scaled logits. It is an upper profile, not a numerical certificate obtained by setting $C_r=1$.

The full-interaction size has the exact gauge interpretation
\begin{equation}
\min_{S\in\mathrm{GL}(d)}\Gamma(QS,KS^{-\top})=\kap(C)
\label{eq:gaugeopt}
\end{equation}
for $\beta>0$, $R>0$, and an ambient realization of the optimal rank-space change of basis. More generally, a native retained subspace $W$ produces a compatible $T_W$ with $\kap(T_W)\le\Gamma_W$ and $q(C-T_W)\le2\eta_W$. Hence the intrinsic profile is termwise no worse than the native theorem at the same declared dimension. The proof of Theorem~\ref{thm:intrinsic} itself does not require \eqref{eq:gaugeopt}; it requires only attainment of $q$ and $\kap$, the perturbation lemma, and the inherited cover.

\section{Size is necessary and the \texorpdfstring{$r/2$}{r/2} exponent is sharp}
\label{sec:sharp}
Interaction dimension and residual alone do not control $r_\eps(A)$ independently of matrix size. Appendix~\ref{app:obstructions} gives rank-one families with zero intrinsic residual and $r_\eps(A)=n$ for arbitrary $n$. This is why $\kap(T)$ is a genuine theorem input rather than a cosmetic descriptor.

\begin{proposition}[Intrinsic finite-grid sharpness]
\label{prop:sharp}
Fix $r\ge1$, $0<\eps<1$, and $L\ge2$. Let $\mathcal G_L=\{1,\ldots,L\}^r$, $n=L^r$, and
\[
s=\log\frac{8r}{1-\eps},\qquad H_{ab}=2s\,a^\top b-s\|b\|_2^2.
\]
Then for $A=\softmax_{\rm row}(H)$ and $C=P_nHP_n$,
\begin{equation}
\rank(C)=r,\qquad \kap(C)=\frac{sr(L-1)^2}{2},\qquad r_\eps(A)=L^r.
\label{eq:gridfacts}
\end{equation}
Consequently $r_\eps(A)=\Theta_{r,\eps}(\kap(C)^{r/2})$, so no smaller uniform size exponent can replace $r/2$ at fixed $r$ and fixed $\eps$.
\end{proposition}

The proof is exact: the rows are truncated discrete Gaussians around their index; $A$ is uniformly closer than $(1-\eps)/2$ to the identity; every $\eps$-approximant is therefore invertible by a Neumann-series argument; and opposite cube vertices certify the exact value of $\kap(C)$. Appendix~\ref{app:sharp-deep} gives every inequality and also records the separate native interaction-exponent sharpness and classical sphere/full-ball laws.

\begin{figure}[t]
\centering
\includegraphics[width=0.73\linewidth]{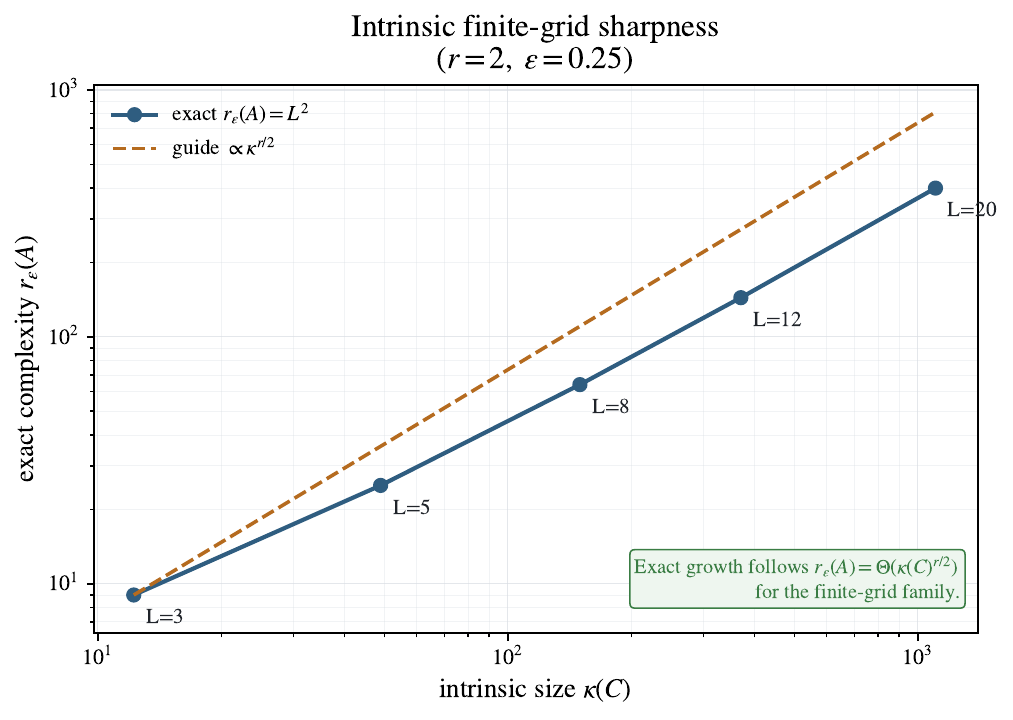}
\caption{Intrinsic finite-grid sharpness at $r=2$ and $\eps=0.25$. Each point is generated from the exact closed forms in Proposition~\ref{prop:sharp}; the dashed curve is a proportionality guide, not a fit.}
\label{fig:sharpness}
\end{figure}

\section{Directly measuring the target on learned attention}
\label{sec:measurement}
Because $C_r$ is unknown, the theorem shape is not used as a plug-in finite rank predictor. We instead bracket the actual target directly. A checked rank-$k$ candidate $B$ with $\|A-B\|_{\rowone}\le\eps$ proves $r_\eps(A)\le k$. Lower witnesses quantify over unrestricted real competitors and include Frobenius-tail, truncated-pseudoinverse, residual-aware, and decoded-inverse criteria. If $\|A-\widehat A\|_{\rowone}\le\chi$, the certified numerical target satisfies
\begin{equation}
L_{\eps+\chi}(\widehat A)\le r_\eps(A)\le U_{\eps-\chi}(\widehat A),
\label{eq:slack}
\end{equation}
with $\eps>\chi$ on the upper side. Failure of a proposed approximant is never used as a lower bound.

The primary learned-head analysis contains 2,302 cells: BERT-base at $256\times256$; GPT-2 and Qwen2.5-1.5B common-prefix blocks at $128\times128$; and two ViT-Base checkpoints at $197\times197$. Text uses 64 pre-specified sources and vision 64 ImageNet-1k validation images. GPT-2 and Qwen2 common-prefix panels retain six of eight pre-specified heads after conditioning on non-local heads; the two ViT checkpoints are never pooled. The final certificate pass covers 3,326 cells at three tolerances for 9,978 brackets. Independent recomputation found zero bracket failures, zero successful falsifications among 458 targeted rank-$(L-1)$ attacks, and 119/119 sampled upper constructions passing the direct maximum-row-$\ell_1$ check.

\begin{table}[t]
\centering\small
\caption{Direct target identification at $\eps=0.25$. $U$ is a checked constructive budget, not a minimum. GPT-2/Qwen2 rows are common-prefix conditional panels.}
\label{tab:brackets}
\begin{tabular}{lrrrr}\toprule
Family & cells & resolved & median $U/L$ & median width/cap\\\midrule
BERT native & 512 & 42.8\% & 3.82 & 0.178\\
GPT-2 CP & 382 & 99.7\% & 3.00 & 0.031\\
Qwen2.5 CP & 384 & 77.1\% & 2.63 & 0.043\\
ViT-Base IN1K & 512 & 59.0\% & 3.41 & 0.094\\
ViT-Base IN21K & 512 & 61.9\% & 3.08 & 0.081\\\bottomrule
\end{tabular}
\end{table}

Four family arms satisfy the pre-specified identification gate on their own, but BERT reaches only 219/512 resolved cells, so the overall gate does not pass. BERT's non-closure is concentrated in four wide heads, and layer 0 resolves 0/128 cells. For GPT-2 and Qwen2, the ordering component depends materially on the decoded-inverse lower witness. At a common $128\times128$ shape, median normalized upper budgets are about $0.14$--$0.17$ for encoder/vision operators versus $0.04$--$0.05$ for the decoder blocks. The two ViT checkpoints order difficulty nearly identically, but ImageNet-1k fine-tuning raises upper endpoints on roughly 82--85\% of cells, so they remain distinct arms.

\begin{figure}[t]
\centering
\includegraphics[width=0.78\linewidth]{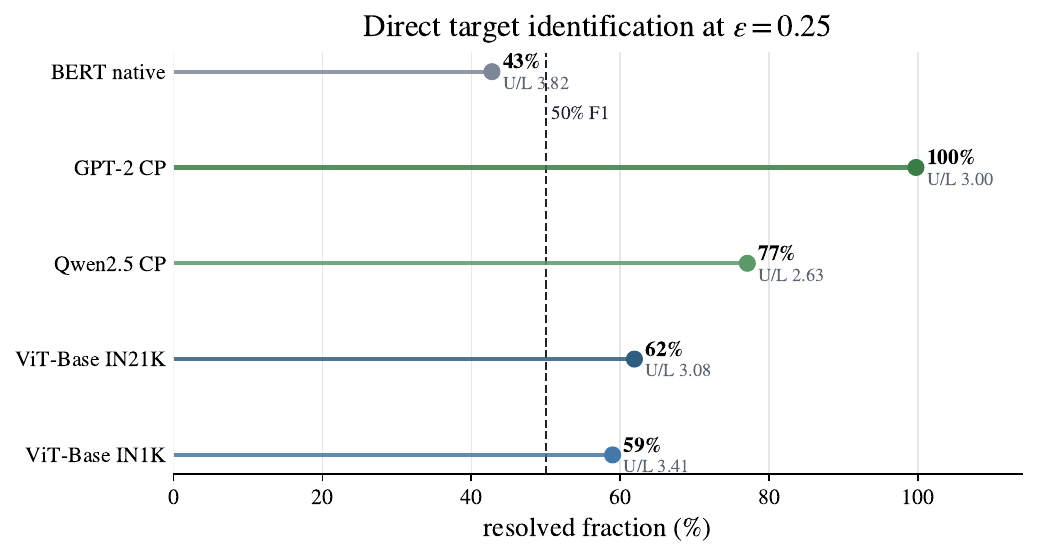}
\caption{Direct target identification at $\eps=0.25$. Points show the resolved fraction; labels report median $U/L$. The dashed 50\% line is one component of the pre-specified F1 criterion. BERT's non-closure is retained rather than removed from the panel.}
\label{fig:brackets}
\end{figure}

An exact rational two-bit oracle provides a known-answer control. Six cells close exactly with ranks $4,3,2,1,3,1$. At $t=s=3/4$ and $\eps=3/4$, unrestricted rank equals one while the minimum row-stochastic rank equals two. Thus restricting the competitor class changes the target, and the unrestricted definition in \eqref{eq:target} is essential. Appendix~\ref{app:certificates} gives the certificate hierarchy and exact control in detail.

\section{A held-out falsification of the historical native geometry}
\label{sec:falsification}
The direct brackets permit a question stronger than theorem validity: does the historical native-coordinate SoftRank block $G$ explain finite constructive difficulty after calibrated alternatives are included? The primary response is $y_U=U_5(\eps=0.25)/c$, $c=\min(m,N)$, a normalized constructive budget rather than an optimum rank. The pre-specified blocks are:
\begin{itemize}[leftmargin=1.5em]
\item \textbf{B0}: structural/nuisance variables such as relative layer depth, domain, and retained-prefix mass;
\item \textbf{B1}: B0 plus weight-space spectra, query/key statistics, double-centred logit structure, calibrated logit-residual proxies, a same-criterion logit-rank proxy, and worst-row-aware concentration descriptors;
\item \textbf{$G$}: historical native SoftRank quantities, including normalized $r_{\rm int}^{\rm SVD}$, $\log(1+\Gamma)$, width-vs-spectrum gap, a prefactor-free native shape score, and a residual-envelope feature;
\item \textbf{B2}: B1 plus attention spectra.
\end{itemize}
Per-family ridge models use source-disjoint outer folds, grouped inner tuning, and source-cluster bootstrap refits. The absolute resolution $\delta=0.02$ was supplemented before real-target results by a relative criterion $\rho_\star=0.10$ because B1 error was already below $\delta$ in several families.

Let $\Delta=\operatorname{MAE}(B1)-\operatorname{MAE}(B1+G)$ and $\rho=\Delta/\operatorname{MAE}(B1)$. Equal-weighting the four primary families, with ViT-Base IN1K as the vision slot, gives
\begin{equation}
\Delta=-0.00005\;[-0.00060,+0.00025],\qquad
\rho=-0.0035\;[-0.0422,+0.0179].
\label{eq:null}
\end{equation}
No family interval excludes zero. At the upper confidence limit, $G$ removes at most about 2\% of the remaining B1 error in aggregate and about 5\% in any family. The stronger B2 comparison gives the same classification. The absolute-scale interpretation retains a headroom qualification because B1 MAE is already below $\delta$ for BERT, GPT-2 CP, and Qwen2 CP.

Against the weak B0 reference, however, $G$ improves aggregate held-out MAE by
\[
\Delta_{B0}=+0.0279\;[+0.0254,+0.0300],
\]
with relative gain about $0.36$. A within-head-permuted smoke analysis retains most of this gain, indicating that it is largely head-level. Within-head B1 comparisons and leave-one-head-out B1 versus B1+$G$ remain null. The scientifically useful conclusion is therefore not that native geometry contains no information, but that it is redundant beyond the calibrated strong baseline for this target and panel.

\begin{figure}[t]
\centering
\includegraphics[width=0.95\linewidth]{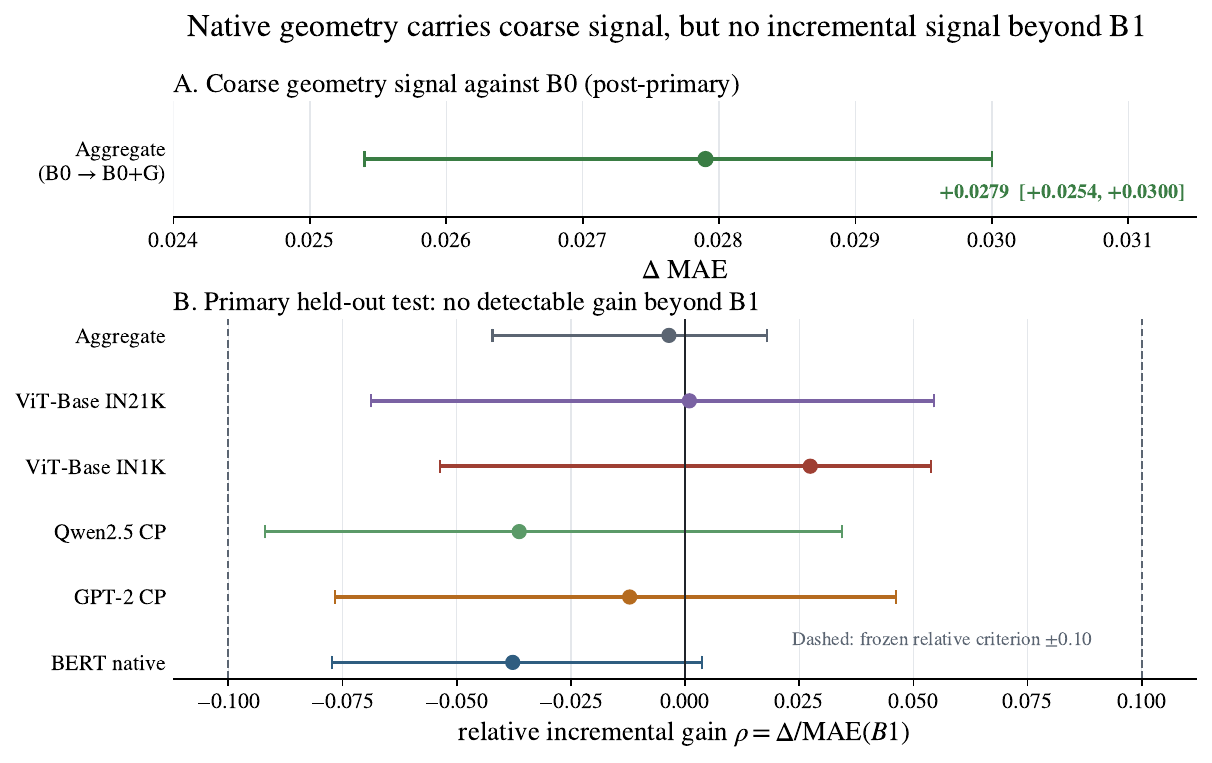}
\caption{Native geometry carries coarse information but no detectable incremental signal beyond B1. Panel A shows the verified post-primary B0$\rightarrow$B0+$G$ aggregate gain. Panel B shows the primary held-out relative gain $\rho$; every 95\% interval crosses zero and remains within the pre-specified relative criterion $\pm0.10$. ViT-Base IN1K is the primary vision arm; IN21K is disclosed separately.}
\label{fig:null}
\end{figure}

Several sensitivity analyses support the same reading: a row-permuted placebo $G$ has gain of the same small order as real $G$ after B1; adding attention spectra does not recover incremental $G$; attention-SVD and domain-transfer targets are null; and diagonal Q/K gauges that preserve $A$ move the native predictors enough to fail the study-resolution gauge-stability audit. These tests concern the learned factorization. The new intrinsic $\kap/q$ descriptor was not extracted on the frozen learned-head panel and its empirical predictive value remains open. Appendix~\ref{app:empirical-deep} gives the full panel contract and sensitivity results.

\section{Discussion and limitations}
The word ``controls'' in the title is theorem-level: the intrinsic profile supplies a certified upper control, and Proposition~\ref{prop:sharp} shows that its size exponent is necessary in the worst case at fixed dimension and tolerance. It does not assert that a scalar descriptor determines instancewise ranks.

The inherited cover constants $C_r$ and $C_d$ are unknown. Setting them to one does not produce a sound finite certificate. Learned-head upper endpoints are checked constructive budgets, not minima, and the optimum remains interval-valued wherever $L<U$. The common-positive-support theorem does not cover a full row-dependent causal mask; the GPT-2 and Qwen2 geometry results use common-prefix blocks conditioned on non-local heads. The panel covers one BERT encoder, GPT-2, Qwen2.5-1.5B, and two ViT-Base checkpoints; no universality beyond those checkpoints and supports is claimed.

Three non-load-bearing theory questions remain open: equality between the globally optimized intrinsic profile and the full native gauge/subspace optimum; joint sharpness in $\kap$, residual margin, and $\eps$; and an efficient global algorithm for the rank-constrained retained-interaction profile. None is needed for Theorem~\ref{thm:intrinsic}. The results concern representation rank, not runtime, memory, sparsity, downstream accuracy, or computational hardness.

\section{Related work}
Efficient-attention methods replace or approximate quadratic softmax attention by low-rank, kernel-feature, sparse, or Nystr\"om structures, including Linformer~\citep{wang2020linformer}, Performer~\citep{choromanski2021performer}, Nystr\"omformer~\citep{xiong2021nystromformer}, and Scatterbrain~\citep{chen2021scatterbrain}. Low-Rank Thinning gives attention guarantees from thinning constructions~\citep{carrell2025lowrank}, while nearly optimal attention coresets study restricted subset representations~\citep{liberty2026coresets}. SoftRank instead fixes the normalized operator and asks for unrestricted matrix rank under maximum-row-$\ell_1$ error.

Geometry can reduce kernel approximation complexity independently of this setting. Altschuler and Parrilo relate kernel rank to algebraic-variety dimension~\citep{altschuler2023kernel}; Budzinskiy studies low-rank approximation of function-generated matrices~\citep{budzinskiy2025bigdata}. Query-key spectra have been used to analyze trained transformers~\citep{pan2024dissecting}; row-centred logit invariants and logit compressibility have also been studied directly~\citep{lee2026invariants,lee2026compressible}. Degrees-of-freedom methods select linear-attention feature dimension~\citep{nishikawa2025dof}. These neighboring objects motivate strong empirical comparators but address different targets or representation classes.

The perturbation argument is closely related to Hilbert/projective probability metrics~\citep{cohen2024hilbert}; the weighted Gibbs cover uses classical exponential-family/Bregman geometry~\citep{banerjee2005bregman} and fixed-dimensional convex-body approximation~\citep{bronshteyn1975polyhedra,gruber1993convex}. The detailed appendix re-derives the exact ingredients used by SoftRank rather than treating them as black boxes.

\section{Conclusion}
We characterized the output-preserving approximation rank of softmax attention through an operator-intrinsic interaction geometry. The profile is gauge invariant, controls the actual maximum-row-$\ell_1$ target, and has a worst-case sharp $r/2$ size exponent at fixed dimension and tolerance. Direct learned-head certificates then separate theorem-level control from a stronger empirical explanation: the historical native-coordinate geometry contains coarse head-level signal but adds no detectable information beyond a strong calibrated baseline on the pre-specified panel. Intrinsic geometry can therefore govern representational complexity without a particular coordinate realization being an empirically unique explanation of learned-head difficulty.

\appendix
\section*{Appendix roadmap}
The appendices collect the complete proof chain and the accepted boundary results underlying the concise main text. Appendix~\ref{app:quotient-deep} derives the exact visible quotient and native robust theorem. Appendix~\ref{app:gibbs-deep} proves the weighted Gibbs cover. Appendix~\ref{app:residual-deep} develops the intrinsic projective residual and sharp softmax perturbation. Appendix~\ref{app:kappa-deep} proves the projective factorization geometry and attainment. Appendix~\ref{app:gauge-deep} gives gauge transformation laws and exact counterexamples. Appendix~\ref{app:profile-deep} develops the jointly constrained intrinsic profile and native domination. Appendices~\ref{app:sharp-deep}--\ref{app:classical-laws} give intrinsic and classical sharpness theorems. Appendix~\ref{app:computable} records computable intrinsic SVD choices as supplementary, non-load-bearing corollaries. Appendices~\ref{app:certificates}--\ref{app:empirical-deep} give the finite certificate hierarchy and learned-head protocol/sensitivities.

\begin{table}[H]
\centering\small
\caption{Dependency map for the detailed appendix.  ``Core'' marks ingredients used directly by the operator-intrinsic theorem; ``interpretive'' marks results that explain its relation to native geometry; ``supporting'' marks sharpness, certificates, and empirical audit material.}
\label{tab:appendix-map}
\begin{tabular}{p{0.19\linewidth}p{0.51\linewidth}p{0.20\linewidth}}
\toprule
Appendix block & Main mathematical role & Role in paper \\
\midrule
\ref{app:quotient-deep} & exact visible quotient, common base weights, native retained theorem & core setup / comparator \\
\ref{app:gibbs-deep} & weighted Gibbs-row cover and radius-product corollary & core proof engine \\
\ref{app:residual-deep} & projective residual and sharp softmax perturbation & core \\
\ref{app:kappa-deep} & factor-radius attainment and projective $\gamma_2$ quotient & core \\
\ref{app:gauge-deep} & gauge transformation laws and exact counterexamples & interpretive boundary \\
\ref{app:profile-deep} & gauge optimum, compatible profile, native domination & core + interpretive \\
\ref{app:sharp-deep}--\ref{app:classical-laws} & intrinsic and classical worst-case constructions & supporting sharpness \\
\ref{app:computable} & balanced/logit-SVD intrinsic choices & supplementary \\
\ref{app:certificates}--\ref{app:empirical-deep} & finite witnesses, exact controls, learned-head audits & supporting evidence \\
\bottomrule
\end{tabular}
\end{table}

\section{Exact quotient, common base weights, and the native robust theorem}
\label{app:quotient-deep}

\subsection{The exact value-output identity}
For completeness, let $M\in\R^{m\times N}$ and let $E$ be a nonzero normed vector space. Then
\begin{equation}
\sup_{\max_j\|v_j\|_E\le1}\max_i\left\|\sum_jM_{ij}v_j\right\|_E
=\max_i\sum_j|M_{ij}|.
\label{eq:app-value}
\end{equation}
The triangle inequality gives the upper bound. For the reverse direction, choose a row $i_\star$ attaining the maximum, choose any unit vector $e\in E$, and set $v_j=\operatorname{sgn}(M_{i_\star j})e$. The $i_\star$ output then has norm $\sum_j|M_{i_\star j}|$. This proves that the target in \eqref{eq:target} is exactly uniform bounded-value output approximation, independent of value dimension or norm.

\subsection{Anchored query/key quotient and common base weights}
Assume native coordinates $H=\beta QK^\top$. Treat individual queries and keys as column vectors and set
\[
u_i=q_i-q_1,\qquad U=\operatorname{span}\{u_i\},\qquad
s_j=P_U(k_j-k_1),\qquad V=\operatorname{span}\{s_j\}.
\]
Then
\begin{equation}
A_{ij}=\frac{\pi_j e^{\beta\langle u_i,s_j\rangle}}
{\sum_\ell\pi_\ell e^{\beta\langle u_i,s_\ell\rangle}},
\qquad \pi_j\propto e^{\beta\langle q_1,k_j\rangle}>0.
\label{eq:anchored-gibbs}
\end{equation}
Indeed, the component of $k_j-k_1$ orthogonal to $U$ pairs to zero with every $u_i$, while the $k_1$ contribution is row scalar. The common query $q_1$ remains as the positive base weights $\pi$. This is the reason common key-only terms are retained rather than treated as a row-softmax symmetry.

Projecting $u_i$ further to $V$ gives effective query factors $x_i=P_Vu_i$. The native exact size is
\[
\rho_Q=\rad\{x_i\},\qquad \rho_K=\rad\{s_j\},\qquad \Gamma=\beta\rho_Q\rho_K.
\]
The exact weighted-cover theorem therefore yields \eqref{eq:nativeexact}.

\subsection{Why the exact visible dimension is \texorpdfstring{$\rank(C)$}{rank(C)}}
For finite positive $A$,
\begin{equation}
\softmax_{\rm row}(H')=\softmax_{\rm row}(H)
\quad\Longleftrightarrow\quad
H'=H+a\one^\top.
\label{eq:softmax-class}
\end{equation}
The reverse direction follows from equality of within-row exponential ratios. Hence $H_A=HP_N=(\log A)P_N$ is a complete representative of the row-softmax class.

Define the anchored cross-difference matrix
\begin{equation}
D_{ij}=H_{ij}-H_{i1}-H_{1j}+H_{11}
=\beta\langle P_Vu_i,s_j\rangle.
\label{eq:Danchor}
\end{equation}
Both factor families span $V$, so $\rank(D)=\dim(V)$ for $\beta>0$. Double centring gives $C=P_mDP_N$. Conversely, because $D$ is anchored, taking the same row/column differences of $C$ recovers $D$. Thus $C$ and $D$ are connected by linear left/right operations in both directions, and
\begin{equation}
R=\dim(V)=\rank(C),\qquad \beta>0.
\end{equation}
At $\beta=0$, $H=C=0$ and $A$ is uniform, so the operator interaction dimension is zero even if an unscaled geometric span was chosen before multiplying by $\beta$.

It is useful to distinguish $HP_N$ from $C$:
\begin{equation}
HP_N=C+\one b_H^\top,
\qquad b_H^\top=\frac1m\one^\top HP_N.
\label{eq:keybias}
\end{equation}
Therefore $\rank(HP_N)$ can be $R$ or $R+1$. The common column term $b_H$ affects the common base weights and cannot be discarded as a softmax equivalence.

\subsection{Native robust retention}
For an actual subspace $W\subseteq U$ with $r=\dim W$ and projector $P=P_W$, define
\[
\xi_i=(I-P)u_i,\qquad \zeta_j=(I-P)s_j,
\]
\begin{align}
\rho_{Q,W}&=\inf_{a_W\in W}\max_i\|Pu_i-a_W\|_2,\nonumber\\
\rho_{K,W}&=\inf_{c_W\in W}\max_j\|Ps_j-c_W\|_2,\qquad
\Gamma_W=\beta\rho_{Q,W}\rho_{K,W},
\end{align}
and
\begin{equation}
\eta_W=\frac\beta2\inf_{a_\perp\in U\cap W^\perp}
\max_i\osc_j\langle\xi_i-a_\perp,\zeta_j\rangle,
\qquad \tau_W=2\tanh(\eta_W/2).
\label{eq:native-residual}
\end{equation}
The centre $a_\perp$ is common across rows. With projected centres $a_W,c_W$ and a residual centre $a_\perp$, orthogonality gives the exact decomposition
\begin{align}
\langle u_i,s_j\rangle={}&
\langle Pu_i-a_W,Ps_j-c_W\rangle
+\langle\xi_i-a_\perp,\zeta_j\rangle\nonumber\\
&+\langle a_W+a_\perp,s_j\rangle
+\langle Pu_i-a_W,c_W\rangle.
\label{eq:native-decomp}
\end{align}
The third term is key-only and enters the common base weights; the fourth is row scalar. Removing the residual bilinear term gives a weighted Gibbs family $\widetilde A$ satisfying $\|A-\widetilde A\|_{\rowone}\le\tau_W$. Applying the full-ball weighted Gibbs cover at the remaining budget $\eps-\tau_W$ yields
\begin{equation}
\tau_W<\eps\quad\Longrightarrow\quad
r_\eps(A)\le\min\left\{p,C_r\left(1+\frac{\Gamma_W}{(\eps-\tau_W)^2}\right)^{r/2}\right\}.
\label{eq:native-robust-app}
\end{equation}
For $r=0$ or $\Gamma_W=0$, rank one is valid under the inclusive boundary $\tau_W\le\eps$.

\subsection{Width and native SVD bridge}
Let
\[
\delta_{Q,W}=\inf_{a_\perp}\max_i\|\xi_i-a_\perp\|_2,
\qquad
\omega_{K,W}=\frac12\max_{j,\ell}\|\zeta_j-\zeta_\ell\|_2.
\]
Then Cauchy--Schwarz gives
\[
\eta_W\le\beta\delta_{Q,W}\omega_{K,W},\qquad
\tau_W\le2\tanh\!\left(\frac{\beta\delta_{Q,W}\omega_{K,W}}2\right).
\]
Mean-centre the visible keys into $Z_{j\cdot}=(s_j-\bar s)^\top$. If $W_r$ is a top-$r$ right-singular subspace and $E_r=Z(I-P_{W_r})$, then
\begin{equation}
\max_j\|E_{r,j\cdot}\|_2\le\sigma_{r+1}(Z),\qquad
\frac12\operatorname{diam}\{E_{r,j\cdot}\}
\le\frac{\sigma_{r+1}(Z)}{\sqrt2}.
\label{eq:native-svd-bridge}
\end{equation}
The second inequality uses $\|e_j-e_\ell\|_2=\sqrt2$. Hence a spectrum-only allowance is
\[
\tau_r^{\rm spec}=2\tanh\!\left(\frac{\beta\rho_{Q,U}\sigma_{r+1}(Z)}{2\sqrt2}\right).
\]
This native SVD construction is output-calibrated, not explained-variance calibrated. It is representation dependent under general Q/K gauge; Appendix~\ref{app:gauge-deep} shows why.

\section{Weighted Gibbs-row cover in full}
\label{app:gibbs-deep}
This section records the inherited rank-control engine in full because the intrinsic theorem uses it as its only non-elementary attention-specific input.

\begin{theorem}[Weighted spherical Gibbs cover]
\label{thm:gibbs-sphere}
Fix $d\ge2$, $0<e<1$, $\gamma\ge0$, keys $k_1,\ldots,k_N\in B_2^d$, and arbitrary positive weights $\pi_j$ summing to one. For $x\in\mathbb S^{d-1}$ define
\[
p_x(j)=\frac{\pi_j e^{\gamma\langle x,k_j\rangle}}
{\sum_{\ell=1}^N\pi_\ell e^{\gamma\langle x,k_\ell\rangle}}.
\]
There is an $\ell_1$ cover by rows from the same family of cardinality
\[
M\le C_d\left(1+\frac{\gamma}{e^2}\right)^{(d-1)/2},
\]
where $C_d$ depends only on $d$, uniformly in $N$ and in the positive base-weight vector $\pi$.
\end{theorem}

\begin{lemma}[Projection-net polytope approximation]
\label{lem:projection-net-polytope}
Let $K\subset RB_2^d$ and choose $\rho>R$. For every $0<\delta\le\rho$ there are boundary points
$z_1,\ldots,z_M\in\partial K$ such that
\[
M\le C_d(\rho/\delta)^{d-1}
\]
and, for $P=\operatorname{conv}\{z_1,\ldots,z_M\}\subset K$,
\[
0\le h_K(u)-h_P(u)\le \frac{\delta^2}{2(\rho-R)}
\qquad\forall u\in\mathbb S^{d-1}.
\]
The constant $C_d$ depends only on $d$.
\end{lemma}
\begin{proof}
Take a chordal $\delta$-net $q_1,\ldots,q_M$ of $\rho\mathbb S^{d-1}$, with
$M\le C_d(\rho/\delta)^{d-1}$, and put $z_s=\Pi_K(q_s)$. Fix $u\in\mathbb S^{d-1}$ and a support point $p\in K$ with $u^\top p=h_K(u)$. The ray $p+tu$ meets $\rho\mathbb S^{d-1}$ at some $q=p+tu$ with $t\ge\rho-R$, and $p=\Pi_K(q)$. Choose $q_s=q+a$ with $\|a\|_2\le\delta$ and write $z_s=p+b$. Projection optimality gives
\[
\|tu+a-b\|_2^2\le\|tu+a\|_2^2,
\]
so
\[
2t\,u^\top b\ge \|a-b\|_2^2-\|a\|_2^2\ge-\delta^2.
\]
Hence
\[
h_K(u)-u^\top z_s=-u^\top b\le\frac{\delta^2}{2t}
\le\frac{\delta^2}{2(\rho-R)}.
\]
Since $h_P(u)\ge u^\top z_s$ and $P\subset K$, the stated two-sided support-gap bound follows.
\end{proof}

\begin{proof}[Proof of Theorem~\ref{thm:gibbs-sphere}]
The case $\gamma=0$ is immediate. For $\gamma>0$ define the log partition
\[
F(\theta)=\log\sum_j\pi_j e^{\langle\theta,k_j\rangle},
\qquad m_x=\nabla F(\gamma x),
\qquad h(x)=2+\frac{F(\gamma x)}\gamma.
\]
Because $\|k_j\|_2\le1$, $-\gamma\le F(\gamma x)\le\gamma$, so $1\le h(x)\le3$. Moreover
\[
\operatorname{KL}(p_x\|\pi)=\gamma\langle x,m_x\rangle-F(\gamma x).
\]
The likelihood ratio $p_x(j)/\pi_j$ is at most $e^{2\gamma}$, hence $0\le\operatorname{KL}(p_x\|\pi)\le2\gamma$. Set
\begin{equation}
a_x=h(x)-\langle x,m_x\rangle
=2-\frac{\operatorname{KL}(p_x\|\pi)}\gamma\in[0,2].
\label{eq:ax}
\end{equation}
On the tangent space to the sphere,
\begin{equation}
\nabla_S^2h(x)+h(x)g_x
=\gamma\operatorname{Cov}_{p_x}(K)|_{T_x\mathbb S^{d-1}}+a_xg_x\succeq0.
\label{eq:support-hessian}
\end{equation}
Extend $h$ one-homogeneously by $H(z)=\|z\|h(z/\|z\|)$ and $H(0)=0$. Away from the origin its Hessian has zero radial component and tangent block $\|z\|^{-1}(\nabla_S^2h+hg)\succeq0$. Segments through the origin are handled by $H(x)+H(-x)>0$. Thus $H$ is convex and is the support function of a compact convex body $K$ satisfying
\begin{equation}
B_2^d\subseteq K\subseteq3B_2^d.
\label{eq:body-sandwich}
\end{equation}
The support point with outer normal $x$ is
\[
c_x=\nabla H(x)=m_x+a_xx.
\]

For the log partition, the Bregman identity gives
\[
D_F(\gamma x,\gamma y)=\operatorname{KL}(p_y\|p_x).
\]
Combining this identity with \eqref{eq:ax} yields the exact support-gap formula
\begin{equation}
h(x)-\langle x,c_y\rangle
=\frac{\operatorname{KL}(p_y\|p_x)}\gamma+a_y(1-\langle x,y\rangle).
\label{eq:support-gap}
\end{equation}
The second term is nonnegative. Therefore a support point whose support gap at $x$ is at most $t$ satisfies $\operatorname{KL}(p_y\|p_x)\le\gamma t$.

Apply Lemma~\ref{lem:projection-net-polytope} to the body in \eqref{eq:body-sandwich} with $R=3$ and $\rho=4$. Given $0<t\le1$, choose $\delta=\sqrt{2t}$. Then the projected net points $z_s$ are exposed boundary points $c_{y_s}$ for outward normals $y_s$, and
\begin{equation}
0\le h_K(x)-h_P(x)\le t\quad\forall x\in\mathbb S^{d-1},
\qquad M\le C_dt^{-(d-1)/2}.
\label{eq:polytope}
\end{equation}
Choose $t=\min\{1,e^2/(2\gamma)\}$. For each $x$, take a vertex attaining $h_P(x)$. Equation~\eqref{eq:support-gap} gives $\operatorname{KL}(p_{y_s}\|p_x)\le e^2/2$, and Pinsker yields $\|p_{y_s}-p_x\|_1\le e$. The small-$\gamma$ and $t<1$ cases combine into the stated $C_d(1+\gamma/e^2)^{(d-1)/2}$ cardinality.

For a finite query set, replace each row by its chosen representative. The resulting matrix has at most $M$ distinct rows, hence rank at most $M$. This is the representative-row-to-rank step used throughout the paper.
\end{proof}

\begin{corollary}[Weighted full-ball lift]
\label{cor:fullball-cover}
If both query and key factors lie in $B_2^r$, the same construction gives
\[
r_e\le\min\left\{p,C_r\left(1+\frac\gamma{e^2}\right)^{r/2}\right\}.
\]
\end{corollary}
\begin{proof}
Lift $x\in B_2^r$ to $(x,\sqrt{1-\|x\|^2})\in\mathbb S^r$ and $y\in B_2^r$ to $(y,0)\in B_2^{r+1}$. Inner products are unchanged, so Theorem~\ref{thm:gibbs-sphere} applies in dimension $r+1$.
\end{proof}

\begin{corollary}[Radius-product weighted Gibbs cover]
\label{cor:radius-product-cover}
Let $x_1,\ldots,x_m$ and $y_1,\ldots,y_N$ be finite point sets in $\R^r$ with Chebyshev radii $R_X,R_Y$, and let $\pi\in\R_{>0}^N$ be any positive common base-weight vector. For
\[
p_i(j)=\frac{\pi_j e^{\langle x_i,y_j\rangle}}
{\sum_{\ell}\pi_\ell e^{\langle x_i,y_\ell\rangle}},
\]
there is a representative-row matrix $B$ satisfying
\[
\|P-B\|_{\rowone}\le e,
\qquad
\rank(B)\le C_r\left(1+\frac{R_XR_Y}{e^2}\right)^{r/2}.
\]
The bound is uniform in $m,N$, and the positive vector $\pi$.
\end{corollary}
\begin{proof}
Translate both factor sets to minimum-enclosing-ball centres. The translation of the query factor changes logits by row constants, while the key translation contributes only a common column term and can therefore be absorbed into the positive base weights. If one radius vanishes, all rows coincide and rank one is exact. Otherwise reciprocally rescale the two factors so both radii become $\sqrt{R_XR_Y}$, divide by this radius to place both factors in the unit ball, and apply Corollary~\ref{cor:fullball-cover} with effective logit scale $\gamma=R_XR_Y$.
\end{proof}

The constants $C_r$ and $C_d$ are symbolic. Fixed-dimensional convex approximation may carry large dimension-dependent prefactors, so these laws should not be converted into finite learned-head certificates by setting $C_r=1$.

\section{Intrinsic projective residual and sharp softmax stability}
\label{app:residual-deep}
Let
\[
\osc_{\rm row}(E)=\max_i(\max_jE_{ij}-\min_jE_{ij}).
\]
A natural but slightly wrong low-rank residual is $\inf_{\rank(M)\le r}\osc_{\rm row}(H-M)$: it is gauge invariant, but its zero threshold is $\rank(HP_N)$ rather than the accepted interaction rank $\rank(C)$. The correction is to allow one free common column term, because such a term becomes a positive base-weight vector in the weighted Gibbs family.

\begin{proposition}[Projective residual]
\label{prop:q}
For any $E\in\R^{m\times N}$,
\begin{equation}
q(E)=\min_{b\in\R^N}\osc_{\rm row}(E-\one b^\top)
=2\min_{a\in\R^m,b\in\R^N}\|E-a\one^\top-\one b^\top\|_{\max}.
\label{eq:q-projective}
\end{equation}
Both minima attain. The kernel is
\[
\Lspace=\{a\one^\top+\one b^\top\},
\]
and $q$ is a norm on $\Cspace=\{E:P_mEP_N=E\}$. On $\Cspace$,
\begin{equation}
\frac12\|E\|_{\max}\le q(E)\le2\|E\|_{\max}.
\label{eq:q-coercive}
\end{equation}
\end{proposition}
\begin{proof}
For fixed $b$, the best row scalar $a_i$ is the midpoint of the range of row $i$, so the row oscillation is twice the smallest maximum absolute residual. To prove attainment, fix the redundant additive gauge by $a_1=0$. A bounded objective sublevel then bounds every $b_j$ using row one and every $a_i$ using any fixed column; closedness gives compactness. Equivalently, \eqref{eq:q-projective} is a finite linear program with constraints
\[
-t/2\le E_{ij}-a_i-b_j\le t/2.
\]
The kernel follows immediately. If $E\in\Cspace$ and $R=E-a\one^\top-\one b^\top$, then $E=P_mRP_N$. Entrywise double centring has operator norm at most four in the maximum-entry norm, while $q(E)=2\min\|R\|_{\max}$; this yields the lower bound in \eqref{eq:q-coercive}. The upper bound follows by taking $a=b=0$.
\end{proof}

The associated intrinsic low-rank residual is
\begin{equation}
d_r(H)=\min_{\substack{T\in\Cspace\\\rank(T)\le r}}q(C-T).
\label{eq:dr}
\end{equation}
The feasible set is closed, and \eqref{eq:q-coercive} makes bounded residual sublevels compact, so the minimum attains. Moreover
\begin{equation}
d_r(H)=0\quad\Longleftrightarrow\quad\rank(C)\le r.
\end{equation}
Thus $d_r$ is intrinsic, nonincreasing in $r$, invariant under row shifts and Q/K gauge, and deliberately invariant under common column shifts because the later cover is uniform over positive base weights.

\begin{lemma}[Sharp softmax perturbation]
\label{lem:softmax-sharp}
For finite $x,h\in\R^N$,
\begin{equation}
\|\softmax(x+h)-\softmax(x)\|_1
\le2\tanh\!\left(\frac{\osc(h)}4\right)
\le\frac{\osc(h)}2.
\label{eq:sharp-softmax-app}
\end{equation}
The first inequality is sharp.
\end{lemma}
\begin{proof}
Let $w=\softmax(x)$, $p=\softmax(x+h)$, and $z_j=p_j/w_j$. Then $\sum_jw_jz_j=1$. Put $a=\min_jz_j$, $b=\max_jz_j$. The likelihood ratio range gives $b/a=R_0=e^{\osc(h)}$ and $a\le1\le b$. Convexity of $|z-1|$ and the chord joining its endpoint values yield
\[
\sum_jw_j|z_j-1|\le\frac{2(1-a)(b-1)}{b-a}.
\]
Substitute $b=R_0a$ and maximize over $a\in[1/R_0,1]$. The maximum occurs at $a=R_0^{-1/2}$ and equals
\[
2\frac{\sqrt{R_0}-1}{\sqrt{R_0}+1}=2\tanh(\osc(h)/4).
\]
Two coordinates attain equality by exchanging the two probabilities. The second inequality is $\tanh u\le u$.
\end{proof}

Combining Proposition~\ref{prop:q} and Lemma~\ref{lem:softmax-sharp}, any retained $T$ satisfies
\begin{equation}
\|A-\softmax_{\rm row}(T+\one b^\top)\|_{\rowone}
\le2\tanh(q(C-T)/4)
\end{equation}
for a minimizing column shift $b$ after reinstating the common base weights. This is the perturbation half of the intrinsic theorem.

\section{Projective factorization geometry: \texorpdfstring{$\kap$, $\gamma_2$}{kappa and gamma2}, and attainment}
\label{app:kappa-deep}
For a fixed matrix $F$, define the Hilbert factorization norm
\begin{equation}
\gammatwo(F)=\inf_{F=UV^\top}\left(\max_i\|u_i\|_2\right)\left(\max_j\|v_j\|_2\right).
\label{eq:gamma2}
\end{equation}
For centred $T$, $\kap(T)$ instead optimizes Chebyshev radii of the two factor point sets.

\subsection{Factor dimension and the Gram SDP}
Any factorization $F=XY^\top$ can be compressed to dimension $\rank(F)$ without increasing either row-norm product or radius product. Project the keys onto the span of the query rows, then project the queries onto the span of the projected keys. Products are unchanged and orthogonal projections cannot increase radii.

The factorization norm has the exact SDP
\begin{equation}
\gammatwo(F)=\min\left\{t:
\begin{pmatrix}G&F\\F^\top&J\end{pmatrix}\succeq0,
\ G_{ii}\le t,\ J_{jj}\le t\right\}.
\label{eq:gamma2-sdp}
\end{equation}
For any nonzero factorization, reciprocal scalar balancing makes the two maximum row norms equal to the square root of their product, giving a feasible Gram matrix. Conversely a feasible positive semidefinite block matrix has Gram factors with both row-norm maxima at most $\sqrt t$. Bounded SDP sublevels are compact because positivity bounds every entry by the diagonal bounds, so the minimum attains. Compression then gives rank-dimensional attainment. In this convention $\gammatwo$ is a norm and
\begin{equation}
\|F\|_{\max}\le\gammatwo(F)\le\sqrt{\min(m,N)}\,\|F\|_{\max}.
\end{equation}

\subsection{Exact quotient identity}
\begin{proposition}[Projective factorization identity]
\label{prop:kappa-quotient}
For every $T\in\Cspace$,
\begin{equation}
\boxed{\kap(T)=\min_{P_mFP_N=T}\gammatwo(F)
=\min_{a,b}\gammatwo(T-a\one^\top-\one b^\top).}
\label{eq:kappa-quotient}
\end{equation}
The radius-product minimum and representative minimum both attain. A minimizing representative can be factored in dimension $\rank(T)$.
\end{proposition}
\begin{proof}
Start from $T=XY^\top$ and translate each factor to an enclosing-ball centre $c,d$. Then
\[
F=(X-\one c^\top)(Y-\one d^\top)^\top
\]
differs from $T$ by an element of $\Lspace$, while
$\gammatwo(F)\le\rad(X)\rad(Y)$. This proves one inequality in \eqref{eq:kappa-quotient} after taking infima.

Conversely, if $P_mFP_N=T$ and $F=UV^\top$, then
\[
T=(P_mU)(P_NV)^\top,
\]
so $\kap(T)\le\rad(U)\rad(V)\le\max_i\|u_i\|\max_j\|v_j\|$. Taking the optimum over $F$ proves the reverse inequality.

It remains to justify attainment of the radius formulation. Starting from any minimizing sequence, subtract factor means, compress to dimension $\rank(T)$, and reciprocally rescale so the two positive radii agree. If a mean-zero point set has enclosing radius $R$ and optimal centre $c$, then
\[
\|c\|=\left\|\frac1m\sum_i(c-x_i)\right\|\le R,
\qquad \max_i\|x_i\|\le2R.
\]
Thus a bounded radius-product sublevel gives bounded factor entries in fixed dimension. A convergent subsequence preserves the product equation and radii by continuity. Representative attainment follows from the already-attained SDP and the same constructions.
\end{proof}

Consequently $\kap$ is a norm on $\Cspace$, with
\begin{equation}
\frac14\|T\|_{\max}\le\kap(T)\le\gammatwo(T)
\le\sqrt p\,\|T\|_{\max},
\label{eq:kappa-basic-bounds}
\end{equation}
and
\begin{equation}
|\kap(T)-\kap(T')|\le\sqrt p\,\|T-T'\|_{\max}.
\end{equation}
For rank-one $T=xy^\top$ with mean-zero factors,
\begin{equation}
\kap(T)=\frac{\operatorname{range}(x)\operatorname{range}(y)}4,
\qquad
\gammatwo(T)=\max_i|x_i|\max_j|y_j|.
\label{eq:rankone-kappa}
\end{equation}
The first formula follows from the rectangular contrast identity
\begin{equation}
|T_{ij}-T_{ik}-T_{\ell j}+T_{\ell k}|
=|\langle x_i-x_\ell,y_j-y_k\rangle|
\le4\rad(X)\rad(Y).
\label{eq:rectcontrast}
\end{equation}
In general, \eqref{eq:rectcontrast} gives the useful lower bounds
\begin{equation}
\kap(T)\ge\frac14\max_{i,\ell,j,k}|\Delta_{i\ell;jk}T|,
\qquad \kap(T)\ge q(T)/2.
\label{eq:kappa-lower}
\end{equation}

\subsection{Why centring matters}
For an arbitrary raw matrix $M$,
\begin{equation}
\inf_{M=XY^\top}\rad(X)\rad(Y)=\kap(P_mMP_N),
\end{equation}
but the raw infimum need not attain. The $2\times2$ matrix
\[
M=\begin{pmatrix}0&1\\1&2\end{pmatrix}
\]
has $P_2MP_2=0$ and infimum zero. Writing its additive parts as $a=b=(0,1)$, append factor coordinates $(a_i/v,v)$ and $(v,b_j/v)$; their radius product tends to zero as $v\to\infty$. An attaining zero product would force equal rows or equal columns, which $M$ lacks. Thus the centred interaction $C$ is not merely convenient: centring is what restores compactness and attainment.

\section{Gauge transformation laws and exact boundary examples}
\label{app:gauge-deep}
The factorization gauge preserves the full logit matrix but changes the Euclidean metric on latent coordinates. This appendix records the exact transformation law rather than relying only on examples.

Write $T_0=S^\top$ in column-vector convention. Then $q_i'=T_0q_i$ and $k_j'=T_0^{-\top}k_j$. Let
\[
\mathcal R=T_0|_U:U\to U'=T_0U.
\]
The anchored query differences and visible key differences transform as
\begin{equation}
u_i'=\mathcal Ru_i,
\qquad
s_j'=\mathcal R^{-*}s_j,
\qquad
V'=\mathcal R^{-*}V.
\label{eq:gauge-visible}
\end{equation}
Indeed, for $u\in U$,
\[
\langle\mathcal Ru,s_j'\rangle
=\langle T_0u,P_{T_0U}T_0^{-\top}(k_j-k_1)\rangle
=\langle u,s_j\rangle.
\]
The visible reprojection is essential when $U$ is a proper subspace; replacing $s'_j$ by the raw ambient $T_0^{-\top}s_j$ can be wrong.

Restrict $L=\mathcal R^{-*}|_V:V\to V'$ and let $D=L^{-*}$. Then the effective exact factors satisfy
\begin{equation}
x_i'=Dx_i,
\qquad s_j'=D^{-*}s_j,
\end{equation}
so
\begin{equation}
\Gamma'=\beta\rad(DX)\rad(D^{-*}Y).
\label{eq:gamma-transform}
\end{equation}
For any invertible $B$,
\[
\sigma_{\min}(B)\rad(F)\le\rad(BF)\le\sigma_{\max}(B)\rad(F),
\]
and therefore
\begin{equation}
\operatorname{cond}_2(D)^{-1}\Gamma
\le\Gamma'\le\operatorname{cond}_2(D)\Gamma.
\label{eq:gamma-cond}
\end{equation}
For visible-key SVD, in orthonormal coordinates on $U,U'$ the mean-centred key matrix transforms as
\begin{equation}
Z'=Z\mathcal R^{-1},
\label{eq:Z-congruence}
\end{equation}
so $Z'^\top Z'$ is obtained by congruence rather than orthogonal similarity. Singular values and right singular subspaces are therefore not gauge invariants.

\subsection{Narrow exact exceptions}
At interaction dimension zero the issue is trivial. At exact interaction dimension one, the exact full product $\Gamma$ is invariant under every invertible gauge. In rank one, the product can be written directly from logits as
\begin{equation}
\Gamma=\frac14\max_{i,\ell,j,k}
|H_{ij}-H_{ik}-H_{\ell j}+H_{\ell k}|,
\label{eq:rankone-gamma}
\end{equation}
which is manifestly gauge invariant. This exception does not extend to $R\ge2$.

\subsection{Counterexample 1: unbounded distortion of exact \texorpdfstring{$\Gamma$}{Gamma}}
Take
\[
Q=K=\begin{pmatrix}0&0\\1&0\\0&1\end{pmatrix},
\qquad S=\operatorname{diag}(t,t^{-1}),\quad t>0.
\]
The logits remain fixed. Both transformed point sets are right triangles with hypotenuse $\sqrt{t^2+t^{-2}}$, so their exact radii are half that value and
\begin{equation}
\Gamma_t=\frac\beta4(t^2+t^{-2})\to\infty.
\label{eq:ce2}
\end{equation}
The interaction rank remains two. Thus even the exact full native shape term can diverge while $A$ and every A-based certificate remain unchanged.

\subsection{Counterexample 2: strict-gap native SVD selection changes}
Let
\[
Q=\begin{pmatrix}0&0\\1&0\\0&1\end{pmatrix},\qquad
K=\begin{pmatrix}-4/3&0\\2/3&-1/2\\2/3&1/2\end{pmatrix},\qquad
S=\operatorname{diag}(4,1).
\]
The key mean is zero and
\[
Z^\top Z=\operatorname{diag}(8/3,1/2),
\qquad
(Z')^\top Z'=\operatorname{diag}(1/6,1/2).
\]
The leading key direction switches from $e_1$ to $e_2$ with strict singular gaps. The logits are exactly unchanged. At rank one the exact residual width changes enough that, at threshold $1/3$,
\begin{equation}
r_{\rm int}^{\rm SVD}(1/3):\quad1\longrightarrow2.
\end{equation}
This failure is not a tie-breaking artifact.

\subsection{Counterexample 3: optimizing over all native orthogonal subspaces still moves}
Take the centrally symmetric sets $Q=K=\{\pm e_1,\pm e_2\}$ and gauge by $S=\operatorname{diag}(a,a^{-1})$, $a\ge1$. For a discarded unit direction $z=(c,s)$, the exact one-dimensional residual objective becomes
\[
\frac{\eta_W}{\beta}
=\max(a|c|,a^{-1}|s|)\max(a^{-1}|c|,a|s|).
\]
By symmetry assume $c\ge s\ge0$ and put $u=s/c$. The objective is
\[
\frac{\max(1,a^2u)}{1+u^2},
\]
whose minimum occurs at $u=a^{-2}$. Hence
\begin{equation}
\inf_{\dim W=1}\eta_W=\beta\frac{a^4}{1+a^4}.
\label{eq:allsubspace-ce}
\end{equation}
It equals $\beta/2$ at $a=1$ and $4\beta/5$ at $a=\sqrt2$. Thus even the best native Euclidean orthogonal retained subspace is representation dependent.

\section{Gauge-optimal size, the jointly constrained intrinsic profile, and native domination}
\label{app:profile-deep}

\subsection{Exact gauge optimization}
\begin{proposition}[Gauge-optimal full-interaction size]
\label{prop:gauge-opt}
Recompute the exact quotient and exact enclosing radii after each invertible Q/K gauge. For $\beta>0$ and $R=\rank(C)>0$,
\begin{equation}
\boxed{\min_{S\in\mathrm{GL}(d)}\Gamma(QS,KS^{-\top})=\kap(C),}
\label{eq:gaugeopt-app}
\end{equation}
and the minimum attains after extending the optimal rank-space change of basis to an ambient invertible map. The zero-interaction branch has both sides zero.
\end{proposition}
\begin{proof}
Every gauged exact quotient factors the same centred $C$ after absorbing $\sqrt\beta$ into both factors and subtracting means, so its radius product is at least $\kap(C)$. Conversely choose optimal rank-$R$ centred factors for $C$ and translate them so their first rows vanish. The native anchored effective factors and these optimal factors are two full-column-rank factorizations of the same rank-$R$ anchored matrix, so they differ by an invertible $R\times R$ change of basis. Realize this change on the visible subspace and extend it by the identity on the complementary latent directions. The corresponding ambient Q/K gauge realizes the optimal radius factors.
\end{proof}

\subsection{Every native retained subspace gives a feasible intrinsic candidate}
For a native retained subspace $W$ of dimension $r$, define the centred retained interaction
\begin{equation}
T_W=P_mD_WP_N,
\qquad (D_W)_{ij}=\beta\langle P_Wu_i,P_Ws_j\rangle.
\end{equation}
The projected centred factors imply $\rank(T_W)\le r$ and
\begin{equation}
\kap(T_W)\le\Gamma_W.
\end{equation}
The residual after subtracting $T_W$ is, modulo additive row/column terms, $\beta\langle\xi_i,\zeta_j\rangle$. Using the native optimal common residual centre as a feasible column shift in $q$ gives
\begin{equation}
q(C-T_W)\le2\eta_W.
\label{eq:native-feasible}
\end{equation}
This proves termwise domination with the exact normalization
$2\tanh((2\eta_W)/4)=2\tanh(\eta_W/2)=\tau_W$.

\subsection{Coupled size-residual optimization}
For $r\ge0$ and $t\ge0$ define
\begin{equation}
K_r(C;t)=\min\{\kap(T):T\in\Cspace,\ \rank(T)\le r,\ q(C-T)\le t\},
\label{eq:Kr}
\end{equation}
with $+\infty$ if infeasible. Feasibility is equivalent to $t\ge d_r(H)$. When feasible, $q(T)\le q(C)+t$ and the coercivity of $q$ bounds the closed feasible set; the rank-at-most-$r$ condition is closed and $\kap$ is continuous. Thus the minimum attains.

With $\ta(t)=2\tanh(t/4)$, any nonzero finite $K_r(C;t)$ with $\ta(t)<\eps$ yields
\begin{equation}
r_\eps(A)\le\min\left\{p,C_r\left(1+\frac{K_r(C;t)}{(\eps-\ta(t))^2}\right)^{r/2}\right\}.
\label{eq:profile-term}
\end{equation}
If $K_r(C;t)=0$, attainment and positivity of $\kap$ force $T=0$ and the rank-one branch applies when $\ta(t)\le\eps$. At $t=0$, feasibility requires $r\ge\rank(C)$ and the only feasible interaction is $C$.

The outer intrinsic profile is the infimum of these valid terms, the degenerate rank-one branch, and the trivial cap $p$. We do not claim that this outer infimum attains because the nondegenerate residual condition is strict. Nor do we assume any monotonicity of the unknown constants $C_r$.

Most importantly, size and residual must arise from the same retained $T$. Combining independently optimized minima is unsound. Appendix~\ref{app:obstructions} gives an exact $2\times2$ example where the separately best residual occurs at $T=C$ and the separately smallest size at $T=0$, while the actual $r_{1/2}(A)$ equals two.

For every native $W$, \eqref{eq:native-feasible} implies
\begin{equation}
K_r(C;2\eta_W)\le\Gamma_W,
\end{equation}
so the intrinsic profile is no worse at the same declared dimension. The intrinsic candidate family is nevertheless strictly broader: arbitrary low-rank attenuations such as $T=\lambda C$ in rank one are not generated by a retain-or-discard native orthogonal subspace. Therefore equality of the globally optimized intrinsic and native profiles remains open.

\section{Sharpness and obstruction constructions}
\label{app:sharp-deep}

\subsection{Full proof of intrinsic finite-grid sharpness}
Let $L\ge2$ be an integer, $\mathcal G_L=\{1,\ldots,L\}^r$, $n=L^r$, $s=\log(8r/(1-\eps))$, and
\[
H_{ab}=2s\,a^\top b-s\|b\|_2^2.
\]
Multiplying each row by the softmax-invisible factor $e^{s\|a\|^2}$ shows
\begin{equation}
A_{ab}\propto e^{-s\|a-b\|_2^2}.
\label{eq:grid-gaussian}
\end{equation}
Hence each row is a product of one-dimensional truncated discrete Gaussians. Put
\[
z=e^{-s}=\frac{1-\eps}{8r}.
\]
For one coordinate, the off-diagonal unnormalized mass relative to the diagonal is bounded by
\[
2\sum_{k\ge1}z^{k^2}\le2\sum_{k\ge1}z^k=\frac{2z}{1-z}.
\]
After normalization, the off-diagonal probability is at most $2z/(1+z)$. A union bound over the $r$ coordinates therefore gives
\begin{equation}
\|A-I_n\|_{\infty\to\infty}
=\|A-I_n\|_{\rowone}
\le\frac{4rz}{1+z}<\frac{1-\eps}{2}.
\label{eq:grid-identity-close}
\end{equation}
If $\|A-B\|_{\rowone}\le\eps$, then
\[
\|I_n-B\|_{\infty\to\infty}
\le\|I_n-A\|_{\infty\to\infty}+\|A-B\|_{\rowone}
<\frac{1+\eps}{2}<1.
\]
Thus $B$ is invertible by the Neumann criterion and $\rank(B)=n$. Since $A$ itself has rank at most $n$,
\begin{equation}
r_\eps(A)=n=L^r.
\end{equation}

Let $\mu=((L+1)/2,\ldots,(L+1)/2)$. Double centring yields
\begin{equation}
C_{ab}=2s(a-\mu)^\top(b-\mu).
\end{equation}
The centred grid spans $\R^r$, so $\rank(C)=r$. The factors
\[
X_a=Y_a=\sqrt{2s}(a-\mu)
\]
form centrally symmetric cubes with radius
\[
\rad(X)=\rad(Y)=\sqrt{2s}\frac{\sqrt r(L-1)}2,
\]
which gives
\[
\kap(C)\le\rad(X)\rad(Y)=\frac{sr(L-1)^2}{2}.
\]
For the reverse inequality, take the two opposite all-coordinate vertices in both the row and column index sets. Their rectangular contrast has magnitude $2sr(L-1)^2$. Equation~\eqref{eq:rectcontrast} therefore gives
\[
\kap(C)\ge\frac14\cdot2sr(L-1)^2=\frac{sr(L-1)^2}{2}.
\]
This proves the exact formula and hence
\[
r_\eps(A)=L^r=\Theta_{r,\eps}(\kap(C)^{r/2}).
\]
The conclusion is fixed-$r$, fixed-$\eps$ size-exponent sharpness. It does not establish joint sharpness in $\kap$, the residual margin, and $\eps$.

\subsection{Perturbation stability of the rank target}
The same induced norm that defines the target gives a useful transfer principle.
\begin{lemma}[Tolerance transfer]
\label{lem:tolerance-transfer}
If $\|A-A^0\|_{\rowone}\le\alpha$, then for every $\eps>0$,
\[
r_\eps(A)\ge r_{\eps+\alpha}(A^0).
\]
If additionally $\alpha<\eps$, then
\[
r_\eps(A)\le r_{\eps-\alpha}(A^0).
\]
\end{lemma}
\begin{proof}
Every rank-$k$ $\eps$-approximant to $A$ is a rank-$k$ $(\eps+\alpha)$-approximant to $A^0$ by the triangle inequality, proving the first relation. Conversely, any rank-$k$ $(\eps-\alpha)$-approximant to $A^0$ is an $\eps$-approximant to $A$.
\end{proof}
The tolerance shifts are necessary: even for a two-state stochastic matrix, approximation rank can drop exactly when the perturbation reaches the available error margin.

\subsection{Fixed-margin robust sharpness with full algebraic interaction}
The intrinsic finite-grid proposition has exact interaction rank $r$. A complementary accepted construction shows that a head can have full algebraic interaction dimension while remaining output-effective at dimension $r$ and still exhibit the same hard exponent.
\begin{theorem}[Robust sharpness under full visible interaction]
\label{thm:robust-full-rank-sharpness}
Fix integers $r\ge1$ and $d\ge r+1$, and fix
\[
0<\eps<1,
\qquad
0<\alpha<\min\{\eps,1-\eps\}.
\]
For all sufficiently large $\beta$ there are bounded query and key lists in $B_2^d$ with $n\asymp_{r,\eps,\alpha}\beta^{r/2}$ such that:
\begin{enumerate}[leftmargin=1.5em]
\item the exact visible interaction dimension is $d$ and the attention matrix has algebraic rank $n$;
\item the top-$r$ visible-key SVD subspace $W$ satisfies $\tau_W^{\rm width}<\alpha$ and $\Gamma_W\asymp_r\beta$;
\item the output-effective subspace and SVD dimensions at tolerance $\alpha$ both equal $r$; and
\item
\[
r_\eps(A)=\Theta_{r,\eps,\alpha}(\Gamma_W^{r/2}).
\]
\end{enumerate}
All constants and the sufficiently-large-$\beta$ threshold may depend on $r,d,\eps,\alpha$. The construction uses independent query and key lists and is not a tied-$Q=K$ theorem.
\end{theorem}
\begin{proof}
Start from the asymmetric discrete-Gaussian hard core used for native interaction sharpness, but instantiate it at tolerance $\eps+\alpha$. Choose its decoder leakage $\delta$ so that
$\eps+\alpha+2\delta<1$. This gives $M=\Theta_{r,\eps,\alpha}(\beta^{r/2})$ message rows and a decoded channel within $2\delta$ of $I_M$.

Embed the keys into $\R^d$ by appending $d-r-1$ sufficiently small linearly independent columns, chosen orthogonal to the constant, retained linear, and quadratic key columns. These appended coordinates do not change the hard-core logits for the unperturbed queries, but make the visible key span full-dimensional. For the product lattice, the first $r$ retained coordinate columns have a strict singular-value gap over the centred quadratic column for sufficiently large $\beta$; scaling the appended columns below that gap makes the top-$r$ visible-key SVD subspace exactly the intended retained coordinate subspace.

Now perturb each query inside an open Euclidean ball of radius
\[
t_\beta=\frac{2\operatorname{arctanh}(\alpha/2)}{\beta}.
\]
Since the keys remain in the unit ball, the induced logit oscillation is bounded by $2\beta t_\beta$, and Lemma~\ref{lem:softmax-sharp} gives
\[
\|A-A^0\|_{\rowone}<\alpha.
\]
Within arbitrarily small open perturbation balls, full affine query span and nonsingularity of the unnormalized exponential matrix are simultaneously generic: the former is open dense once $n\ge d+1$, while the determinant is a nonzero real-analytic function because one can realize a Vandermonde specialization with distinct exponential nodes. Hence the perturbation can be chosen so that the visible query span is $\R^d$ and $\rank(A)=n$.

Relative to the intended $r$-dimensional retained subspace, the residual query cloud lies within $t_\beta$ of a common centre. The complementary key half-diameter is bounded by an absolute constant, so
\[
\tau_W^{\rm width}
\le2\tanh(\beta t_\beta/4)<\alpha.
\]
The retained query/key radii stay bounded above and below by positive constants, giving $\Gamma_W\asymp_r\beta$.

By Lemma~\ref{lem:tolerance-transfer}, the hard-core decoded lower bound at tolerance $\eps+\alpha$ transfers to the perturbed matrix, so $r_\eps(A)=\Omega(\beta^{r/2})$. The native robust theorem \eqref{eq:native-robust-app}, with residual below $\alpha$, gives the matching $O(\Gamma_W^{r/2})=O(\beta^{r/2})$ upper order.

Finally, suppose some subspace of dimension $s<r$ had residual allowance at most $\alpha$. Bounded factors give $\Gamma=O(\beta)$, and the robust upper theorem would imply $r_\eps(A)=O(\beta^{s/2})$, contradicting the lower bound for sufficiently large $\beta$. Thus the optimal output-effective subspace dimension is at least $r$; the displayed $W$ shows it is at most $r$. Any earlier SVD cutoff would itself provide a feasible lower-dimensional subspace, so the SVD-effective dimension also equals $r$.
\end{proof}
\begin{remark}
The theorem is a fixed-error, fixed-margin worst-case existence result. It does not establish joint sharpness as $\eps-\tau\downarrow0$, does not control the conditioning of the full-rank attention matrix, and does not yield a universal per-head lower law from $\Gamma_W$ alone.
\end{remark}

\section{Size-term and compatibility obstructions}
\label{app:obstructions}

\subsection{Interaction rank one with zero residual can still force full rank}
Set $r=1$ in the previous construction. For $i,j\in\{1,\ldots,n\}$,
\[
H_{ij}=2sij-sj^2,
\qquad s=\log\frac8{1-\eps}.
\]
The raw logit matrix has rank at most two, the doubly centred interaction has rank one, and $d_1(H)=0$. Yet the same discrete-Gaussian argument yields $r_\eps(A)=n$ and
\[
\kap(C)=\frac{s(n-1)^2}{2}.
\]
Thus no bound depending only on interaction rank, residual, and fixed $\eps$ can be independent of matrix size.

The obstruction can even be realized with raw logit rank one. Choose an integer $J\ge8/(1-\eps)$. For key group $j\in\{1,\ldots,n\}$, repeat the key $N_j=J^{n^2-j^2}$ times and set
\[
H_{i,(j,a)}=2(\log J)ij.
\]
This matrix has raw rank one. If $D$ decodes each repeated key to its group label, then $(AD)_{ij}\propto J^{-(i-j)^2}$ and $\|AD-I_n\|_{\infty\to\infty}<(1-\eps)/2$. Any $\eps$-approximant $B$ makes $BD$ invertible, so $\rank(B)\ge n$.

\subsection{Separately minimizing size and residual is invalid}
Take
\[
C=H=\begin{pmatrix}1&-1\\-1&1\end{pmatrix},\qquad \eps=1/2.
\]
Over rank-at-most-one retained matrices, the residual minimum is zero at $T=C$, whereas the size minimum is zero at $T=0$. Combining these two independent minima would incorrectly suggest the degenerate rank-one branch. In fact, with $v=\tanh1>1/2$,
\[
A=\frac12\begin{pmatrix}1+v&1-v\\1-v&1+v\end{pmatrix},
\qquad \|A^{-1}\|_{\infty\to\infty}=1/v.
\]
For every $B$ with $\|A-B\|_{\rowone}\le1/2$,
\[
\|A^{-1}(A-B)\|_{\infty\to\infty}<1,
\]
so $B$ is invertible and $r_{1/2}(A)=2$. This is why $K_r(C;t)$ couples size and residual through the same $T$.

\subsection{The base-uniform profile is not a complete statistic for \texorpdfstring{$A$}{A}}
Let
\[
H_c=c\begin{pmatrix}1&-1\\-1&1\end{pmatrix}+\one(4c,0),\qquad c>0.
\]
The interaction is always
\[
C_c=c\begin{pmatrix}1&-1\\-1&1\end{pmatrix},\qquad \kap(C_c)=c.
\]
But the common column bias makes the first key dominate by at least $2c$ in every row. The rank-one matrix with every row $(1,0)$ has error at most $2e^{-2c}$, so fixed-error rank one is feasible for large $c$ even while $\kap(C_c)\to\infty$. At $c=1$, the unbiased matrix above has target rank two while a suitable common column bias yields rank one. Therefore the intrinsic profile is a valid base-uniform upper-control object, not a complete invariant of the rank target.

\subsection{Rectangular contrasts lower-bound but do not equal the residual norm}
For
\[
E=\frac12\begin{pmatrix}1&-1&0\\0&1&-1\\-1&0&1\end{pmatrix},
\]
$E$ is doubly centred and $q(E)=1$. Indeed each row has oscillation one, while for any common column shift three selected row differences sum to three and each is bounded by the maximum shifted row oscillation. The largest rectangular contrast, however, has magnitude $3/2$, whose half is only $3/4$. Thus rectangular contrasts provide useful lower bounds but do not generally replace the projective LP.

\section{Classical support-geometry laws and native interaction sharpness}
\label{app:classical-laws}
The intrinsic theorem rests on a cover whose fixed-dimensional exponent also admits matching worst-case constructions in important native geometries. These laws are not the headline of the present paper, but they explain the geometric origin of the $r/2$ exponent.

\subsection{Sharp spherical law}
For $X=(x_1,\ldots,x_n)\in(\mathbb S^{d-1})^n$, let $A_\beta(X)_{ij}\propto e^{\beta\langle x_i,x_j\rangle}$ and define
\[
\mathfrak R_\eps^{\rm sph}(n,d,\beta)=\sup_Xr_\eps(A_\beta(X)).
\]
For fixed $d\ge2$ and $0<\eps<1$,
\begin{equation}
\mathfrak R_\eps^{\rm sph}(n,d,\beta)
\asymp_{d,\eps}\min\{n,(1+\beta)^{(d-1)/2}\}.
\label{eq:spherical-law}
\end{equation}
The upper bound is Theorem~\ref{thm:gibbs-sphere} with uniform weights. For the lower bound use
\[
w_\beta(x,y)=e^{-\beta(1-\langle x,y\rangle)},
\qquad I_d(\beta)=\int_{\mathbb S^{d-1}}w_\beta(x,y)\,d\sigma(y).
\]
Rotational invariance makes $I_d$ independent of $x$, and elementary Laplace bounds give $I_d(\beta)\asymp_d\beta^{-(d-1)/2}$ for $\beta\ge1$. Draw $N$ independent uniform sphere points. Some realization has total directed cross-energy at most twice its expectation; discarding high-energy vertices leaves $M\asymp_{d,\eps}\beta^{(d-1)/2}$ centres for which every row places all but a chosen small mass on its own centre. If $n>M$, duplicate centres into balanced clusters. The resulting attention matrix is within $\delta<1-\eps$ of a rank-$M$ block-averaging projection $P$. Selecting one row per cluster and decoding clusters gives
\[
\|SBD-I_M\|_{\infty\to\infty}\le\eps+\delta<1
\]
for every $\eps$-approximant $B$, so $SBD$ is invertible and $\rank(B)\ge M$.

\subsection{Full-ball law}
For tokens in $B_2^d$, the full-ball lift gives
\[
\mathfrak R_\eps(n,d,\beta)\le\min\left\{n,C'_d\left(1+\frac\beta{\eps^2}\right)^{d/2}\right\}.
\]
A finite discrete-Gaussian construction gives the complementary lower bound
\begin{equation}
\mathfrak R_\eps(n,d,\beta)
\ge c_{d,\eps}\min\left\{n,
\left(1+\min\left\{\beta,\log_+\frac n{C_d}\right\}\right)^{d/2}\right\}.
\label{eq:fullball-finite}
\end{equation}
The construction places lattice points $z_u=u/\sqrt\beta$ in a small box, assigns multiplicities proportional to $\exp((\beta d b^2-\|u\|^2)/2)$, and selects message queries whose induced key distribution is a translated discrete Gaussian. Nearest-message decoding again produces a channel within fixed row-$\ell_1$ distance of identity, so any sufficiently accurate low-rank approximant must have rank at least the message count. When $n\ge C_de^{\beta/8}$ and $\beta$ is large enough,
\begin{equation}
\mathfrak R_\eps(n,d,\beta)\asymp_{d,\eps}\beta^{d/2}.
\label{eq:fullball-law}
\end{equation}
The exponential token requirement in this tied self-attention lower law is essential to the stated construction.

\subsection{Full proof of the finite full-ball lower certificate}
Fix $d\ge1$ and $0<\eps<1$. Choose an even integer $L=L(d,\eps)$ so that for a standard discrete Gaussian $G\propto e^{-\|u\|_2^2/2}$ on $\mathbb Z^d$,
\begin{equation}
\Pr(\|G\|_\infty\ge L/2)\le\delta/4,
\qquad \eps+2\delta<1.
\label{eq:dg-tail}
\end{equation}
Take $0<b\le1/(2\sqrt d)$ and lattice groups
\[
z_u=\frac{u}{\sqrt\beta},\qquad \|u\|_\infty\le b\sqrt\beta.
\]
All tokens lie in $(1/2)B_2^d$. Assign ideal multiplicities
\[
w_u=\exp\!\left(\frac{\beta d b^2-\|u\|_2^2}{2}\right)\ge1
\]
and integer multiplicities $m_u=\lceil w_u\rceil$, so $w_u\le m_u\le2w_u$. The base token count satisfies
\begin{equation}
N_0=\sum_um_u\le C_de^{\beta db^2/2}.
\label{eq:N0}
\end{equation}
For arbitrary $n\ge N_0$, take $T=\lfloor n/N_0\rfloor$ whole copies and distribute the remaining prefix. Then
\[
Tw_u\le m_u'\le4Tw_u.
\]
Select message indices $v$ whose success cubes remain in the lattice box and use query $q_v=z_{Lv}$. For the ideal weights,
\begin{align}
w_u e^{\beta\langle q_v,z_u\rangle}
&=\exp\!\left(\frac{\beta db^2}{2}+\frac{L^2\|v\|^2}{2}\right)
 e^{-\|u-Lv\|^2/2}.
\label{eq:fullball-gaussian}
\end{align}
Thus, up to a factor independent of $u$, the row is a translated standard discrete Gaussian. Integer rounding changes probability ratios by at most a universal factor. Equation~\eqref{eq:dg-tail} therefore bounds the probability of leaving the nearest-message Voronoi cell by $\delta$. Truncation at the lattice boundary cannot increase the failure probability because every selected success cube lies completely inside the box.

There are
\[
M\asymp_d(b\sqrt\beta/L)^d
\asymp_{d,\eps}(\beta b^2)^{d/2}
\]
message rows. Decode each token group to its nearest message. The selected $M\times M$ channel is within row-$\ell_1$ distance $2\delta$ of $I_M$. Selection and decoding are contractions, hence any $\eps$-approximant $B$ obeys
\[
\|SBD-I_M\|_{\infty\to\infty}\le\eps+2\delta<1.
\]
Therefore $SBD$ is invertible and $\rank(B)\ge M$. Finally take
\[
b^2=\min\left\{\frac1{4d},\frac2{\beta d}\log\frac n{C_d}\right\},
\]
which is compatible with \eqref{eq:N0} and yields \eqref{eq:fullball-finite}. If $n\ge C_de^{\beta/8}$, take $b=1/(2\sqrt d)$ to obtain the tied full-ball law \eqref{eq:fullball-law}.

\subsection{Native interaction-exponent sharpness: full construction}
Fix $r\ge1$ and $0<\eps<1$. Choose an even $L=L(r,\eps)$ and $\delta>0$ so that a standard discrete Gaussian on $\mathbb Z^r$ satisfies
\[
\Pr(\|G\|_\infty\ge L/2)\le\delta,
\qquad \eps+2\delta<1.
\]
Set $b=1/(2\sqrt r)$ and $a=\lfloor b\sqrt\beta\rfloor$. Index keys by $u\in\mathbb Z^r$ with $\|u\|_\infty\le a$ and define
\begin{equation}
k_u=\left(\frac u{\sqrt\beta},-\frac{\|u\|_2^2}{\beta}\right)\in\R^{r+1}.
\end{equation}
Index message queries by $v$ satisfying $\|Lv\|_\infty\le a-L/2$ and set
\begin{equation}
q_v=\left(\frac{Lv}{\sqrt\beta},\frac12\right).
\end{equation}
There are $N=(2a+1)^r\asymp_r\beta^{r/2}$ keys and $M\asymp_{r,\eps}\beta^{r/2}$ messages. Pad the query list to $N$ by repeating $q_0$. Since $\|u\|^2/\beta\le1/4$ and the message box is similarly bounded, all vectors lie in $B_2^{r+1}$ for large enough $\beta$.

For a message row,
\begin{align}
\beta\langle q_v,k_u\rangle
&=Lv\cdot u-\frac{\|u\|^2}{2}\\
&=\frac{L^2\|v\|^2}{2}-\frac{\|u-Lv\|^2}{2}.
\end{align}
Hence each message distribution is exactly a translated discrete Gaussian conditioned on the lattice box. Its success cube lies inside the box and carries probability at least $1-\delta$. Nearest-message decoding gives an $M\times M$ channel satisfying
\[
\|SAD-I_M\|_{\infty\to\infty}\le2\delta.
\]
For every $\eps$-approximant $B$,
\[
\|SBD-I_M\|_{\infty\to\infty}\le\eps+2\delta<1,
\]
so $\rank(B)\ge M\asymp_{r,\eps}\beta^{r/2}$.

Ordering $q_0,k_0$ first, the query-difference span is $\R^r\times\{0\}$ and the projected key differences span the same space. Thus the exact visible interaction dimension is $r$. The projected point sets are symmetric lattice boxes with radii bounded above and below by positive constants depending only on $r$, so $\Gamma=\beta\rho_Q\rho_K\asymp_r\beta$. Therefore
\[
r_\eps(A)\ge c_{r,\eps}\beta^{r/2}\asymp_{r,\eps}\Gamma^{r/2}.
\]
This proves native fixed-$r$, fixed-$\eps$ minimax sharpness. It is distinct from the intrinsic finite-grid proposition, which is stated directly in terms of $\kap(C)$ and survives exact gauge optimization.

\section{Computable intrinsic choices: balanced factors and logit SVD}
\label{app:computable}
This section gives theorem-sufficient intrinsic choices that avoid the global rank-constrained optimization in $K_r(C;t)$. They are supplemental and non-load-bearing for Theorem~\ref{thm:intrinsic}: no optimality of the balanced construction is claimed, and the global efficient-optimization problem remains open. The algebraic claims below follow directly from the displayed Gram matrices, spectral residual bound, and explicit counterexamples; they are retained as computable corollaries rather than as headline theory.

\subsection{Balanced factors of a fixed retained interaction}
For a fixed centred $T$ with compact SVD $T=L\Sigma R^\top$, define
\[
X_{\rm bal}=L\Sigma^{1/2},\qquad
Y_{\rm bal}=R\Sigma^{1/2},\qquad
\kap_{\rm bal}(T)=\rad(X_{\rm bal})\rad(Y_{\rm bal}).
\]
This size is intrinsic to $T$ even when nonzero singular values repeat: simultaneous rotations within a repeated block preserve all pairwise distances. Equivalently the two factor Gram matrices are uniquely
\[
G_Q=(TT^\top)^{1/2},\qquad G_K=(T^\top T)^{1/2}.
\]
Thus the finite metric point sets are determined up to isometry.

Using the origin as a feasible centre gives the simpler bound
\[
g_0(T)=\sqrt{\max_i(G_Q)_{ii}\ \max_j(G_K)_{jj}},
\]
and
\begin{equation}
\kap(T)\le\kap_{\rm bal}(T)\le g_0(T)\le\|T\|_2.
\label{eq:balanced-chain}
\end{equation}
Any of these upper sizes can be inserted into Theorem~\ref{thm:intrinsic} together with the residual of the same $T$.

Balanced factors need not approximate the optimal $\kap$ within a rank-only constant. Fix $L\ge1$. Let the query factor contain $L$ copies of each of $\pm e_1$ and one copy of each of $\pm e_2$, while the key factor reverses those multiplicities. For $C=QK^\top$, both native radii equal one and an $e_1$ rectangular contrast gives $\kap(C)=1$. The two nonzero singular values of $C$ are both $2\sqrt L$, so the balanced factors have radius $L^{1/4}$ on both sides and
\begin{equation}
\rank(C)=2,
\qquad \kap(C)=1,
\qquad \kap_{\rm bal}(C)=\sqrt L.
\label{eq:balanced-gap}
\end{equation}
Thus ``canonical'' and ``theorem-sufficient'' do not mean ``optimal.''

\subsection{A logit-SVD intrinsic family}
Apply a fixed matrix-only SVD truncation rule directly to $C$:
\[
T_r=\sum_{a\le r}\sigma_a l_ar_a^\top.
\]
Nonzero singular vectors lie in the centred row/column spaces, so $T_r\in\Cspace$. Define
\[
t_r^0=\osc_{\rm row}(C-T_r),\qquad t_r^*=q(C-T_r).
\]
Because $q$ optimizes a common column shift, $t_r^*\le t_r^0$. Moreover
\begin{equation}
t_r^*\le t_r^0\le\sqrt2\,\sigma_{r+1}(C),
\qquad
\ta_r^{\rm logit\text{-}spec}
=2\tanh\!\left(\frac{\sigma_{r+1}(C)}{2\sqrt2}\right).
\label{eq:logit-svd-residual}
\end{equation}
Indeed every residual row difference is
$e_i^\top(C-T_r)(e_j-e_k)$, whose magnitude is at most
$\|C-T_r\|_2\|e_j-e_k\|_2=\sqrt2\sigma_{r+1}(C)$.

Pairing one allowance from \eqref{eq:logit-svd-residual} with one valid size upper bound from \eqref{eq:balanced-chain} gives a fully intrinsic, computable theorem-sufficient family. It does not solve the nonconvex global retained-matrix optimization and its prefactor-free log expression remains only a shape diagnostic because $C_r$ is unknown.

Repeated singular values require an explicit matrix-only convention. For
\[
C=P_3=I_3-\one\one^\top/3,
\]
the two unit vectors
\[
v=(1,-1,0)/\sqrt2,
\qquad w=(1,1,-2)/\sqrt6
\]
produce equally valid best rank-one truncations $T_v=vv^\top$ and $T_w=ww^\top$, but
\[
\kap(T_v)=1/2,\qquad \kap(T_w)=3/8,
\qquad q(C-T_v)=3/4,\qquad q(C-T_w)=1.
\]
A deterministic tie rule or whole-cluster convention is therefore part of any canonical implementation.

\section{Finite certificate hierarchy and exact small-regime control}
\label{app:certificates}
The empirical brackets use constructive upper witnesses and lower witnesses that quantify over all unrestricted real competitors. This section states the operative certificate mathematics explicitly.

\subsection{Constructive upper witnesses}
\begin{lemma}[Checked upper witness]
If an explicit matrix $B$ has $\rank(B)\le k$ and $\|A-B\|_{\rowone}\le\eps$, then $r_\eps(A)\le k$.
\end{lemma}
This implication is tautological but operationally important: failure of a search procedure to find such a $B$ is never interpreted as a lower bound.

\subsection{Frobenius-tail lower witness}
Let
\[
F_k(A)=\min_{\rank(B)\le k}\|A-B\|_F
=\left(\sum_{j>k}\sigma_j(A)^2\right)^{1/2}.
\]
\begin{proposition}[Frobenius-tail certificate]
\label{prop:frobenius-cert}
If $F_k(A)>\sqrt m\,\eps$, then $r_\eps(A)\ge k+1$.
\end{proposition}
\begin{proof}
For every error matrix $E$,
\[
\|E\|_F^2=\sum_i\|E_{i\cdot}\|_2^2
\le\sum_i\|E_{i\cdot}\|_1^2
\le m\|E\|_{\rowone}^2.
\]
Thus any rank-$k$ matrix within row-$\ell_1$ error $\eps$ would have Frobenius error at most $\sqrt m\,\eps$, contradicting the definition of $F_k(A)$.
\end{proof}

\subsection{Input-subspace reconstruction certificates}
\begin{proposition}[Exact reconstruction certificate]
\label{prop:reconstruction-cert}
Let $Q\in\R^{N\times k}$ have full column rank and let $T\in\R^{N\times m}$ satisfy
\[
TAQ=Q.
\]
If
\[
\eps\|T\|_{\infty\to\infty}<1,
\]
then $r_\eps(A)\ge k$.
\end{proposition}
\begin{proof}
If a rank-$(k-1)$ competitor $B$ existed, $BQ$ would have a nonzero null vector: choose $x\ne0$ with $BQx=0$ and put $v=Qx\ne0$. Then
\[
v=TAv=T(A-B)v,
\]
so
\[
\|v\|_\infty
\le\|T\|_{\infty\to\infty}\|A-B\|_{\infty\to\infty}\|v\|_\infty
\le\eps\|T\|_{\infty\to\infty}\|v\|_\infty,
\]
a contradiction.
\end{proof}

\begin{corollary}[Truncated-pseudoinverse certificate]
\label{cor:pinv-cert}
Let $A=U\Sigma V^\top$ and suppose $\sigma_k>0$. With
\[
T_k=V_k\Sigma_k^{-1}U_k^\top,
\]
$T_kAV_k=V_k$. Hence
\[
\eps\|T_k\|_{\infty\to\infty}<1
\quad\Longrightarrow\quad
r_\eps(A)\ge k.
\]
The relevant norm is the induced infinity norm, i.e. the maximum absolute row sum of the $N\times m$ matrix $T_k$; it is not merely $1/\sigma_k$. The sequence $\|T_k\|_{\infty\to\infty}$ need not be monotone in $k$, so certificate scans do not stop at the first failed dimension.
\end{corollary}

\begin{proposition}[Residual-aware reconstruction certificate]
\label{prop:residual-reconstruction}
Let $Q\in\R^{N\times k}$ have full column rank, let $T\in\R^{N\times m}$ be any proposed reconstruction, and set
\[
R=Q-TAQ,
\qquad
\kappa_0=\frac{\sqrt N\,\|R\|_{2\to\infty}}{\sigma_{\min}(Q)},
\qquad
\|R\|_{2\to\infty}=\max_j\|R_{j\cdot}\|_2.
\]
If
\[
\kappa_0+\eps\|T\|_{\infty\to\infty}<1,
\]
then $r_\eps(A)\ge k$.
\end{proposition}
\begin{proof}
If $BQx=0$ for a nonzero $x$, put $v=Qx$. From the definition of $R$,
\[
v=T(A-B)v+Rx.
\]
Moreover
\[
\|Rx\|_\infty\le\|R\|_{2\to\infty}\|x\|_2
\le\frac{\sqrt N\,\|R\|_{2\to\infty}}{\sigma_{\min}(Q)}\|v\|_\infty.
\]
The strict displayed margin then contradicts $v\ne0$.
\end{proof}

\subsection{Contractive decoding and decoded inverse witnesses}
\begin{lemma}[Contractive decoder reduction]
\label{lem:decoder-reduction}
Let $S$ select rows and let $D\in\R^{N\times h}$ satisfy
$\|D\|_{\infty\to\infty}\le1$. Put $C=SAD$. For every $B$ with
$\|A-B\|_{\rowone}\le\eps$,
\[
\|C-SBD\|_{\rowone}\le\eps,
\qquad
\rank(SBD)\le\rank(B).
\]
Consequently any lower certificate for $C$ at tolerance $\eps$ is also a lower certificate for $A$.
\end{lemma}
\begin{proof}
Row selection has induced infinity norm one, so
\[
\|S(A-B)D\|_{\infty\to\infty}
\le\|S\|\,\|A-B\|_{\infty\to\infty}\,\|D\|_{\infty\to\infty}
\le\eps.
\]
The rank inequality follows from multiplication by $S$ and $D$.
\end{proof}

\begin{corollary}[Decoded inverse certificate]
\label{cor:decoded-inverse}
Assume $C=SAD\in\R^{s\times s}$ is invertible and $\|D\|_{\infty\to\infty}\le1$. If
\[
\eps\|C^{-1}\|_{\infty\to\infty}<1,
\]
then $r_\eps(A)\ge s$.
More generally, for any proposed inverse $X$, the numerically safe condition
\[
\rho+\eps\|X\|_{\infty\to\infty}<1,
\qquad
\rho=\|I-XC\|_{\infty\to\infty},
\]
has the same conclusion.
\end{corollary}
\begin{proof}
Lemma~\ref{lem:decoder-reduction} gives $SBD=C+E$ with $\|E\|_{\infty\to\infty}\le\eps$. In the exact inverse case,
\[
C+E=C(I+C^{-1}E)
\]
is invertible by the Neumann criterion. For a proposed inverse $X$,
\[
I-X(C+E)=(I-XC)-XE
\]
has norm strictly below one, so $X(C+E)$ and hence $C+E$ are injective on $\R^s$. Thus $\rank(B)\ge s$.
\end{proof}

The near-identity condition
\[
\|C-I_s\|_{\infty\to\infty}+\eps<1
\]
is the special case $X=I_s$. A singular-value sufficient condition is
$\sigma_{\min}(C)>\sqrt{s}\,\eps$, because
$\|E\|_2\le\|E\|_F\le\sqrt{s}\eps$.

\subsection{Numerical target slack}
\begin{corollary}[Slack transfer]
\label{cor:numerical-slack}
Suppose a reconstructed operator $\widehat A$ satisfies
$\|A-\widehat A\|_{\rowone}\le\chi$. Then
\[
L_{\eps+\chi}(\widehat A)\le r_\eps(A)
\le U_{\eps-\chi}(\widehat A),
\]
where the upper statement requires $\eps>\chi$ and $L_t,U_t$ denote any sound lower/upper certificate endpoints at tolerance $t$.
\end{corollary}
\begin{proof}
Apply Lemma~\ref{lem:tolerance-transfer} to $A$ and $\widehat A$, then use the defining validity of the lower and upper certificate endpoints.
\end{proof}
The confirmatory extraction records reconstruction slack on the order of $10^{-7}$--$10^{-6}$ in row-$\ell_1$, far below the paper's target tolerances; the implementation nevertheless charges it explicitly rather than dropping it.

\subsection{Exact two-bit oracle}
A six-cell rational control closes exact unrestricted ranks $4,3,2,1,3,1$. One cell illustrates why the competitor class matters. At $t=s=3/4$ and $\eps=3/4$, the unrestricted target is rank one but the minimum row-stochastic rank is two; the best row-stochastic rank-one error is $33/32>\eps$. At $\eps=5/8$, the Frobenius and rank-three truncated-pseudoinverse tests fail, while a stronger explicit reconstruction witness with norm $4/3$ and strict margin $1/6$ combines with a rank-three upper witness of error $9/16$ to close the bracket at $[3,3]$. These exact controls validate the target semantics and the ordering of certificate strength; they do not estimate $C_r$ and do not imply learned-head tightness.

\section{Learned-head protocol, bracket integrity, and sensitivity analyses}
\label{app:empirical-deep}

\subsection{Panel}
\begin{table}[H]
\centering\small
\caption{Primary learned-head panel. CP denotes common-prefix conditional attention.}
\begin{tabular}{lllll}\toprule
Family & checkpoint & operator & cells & layers\\\midrule
BERT & \texttt{bert-base-uncased} & $256\times256$ & 512 & 0,4,7,11\\
GPT-2 CP & \texttt{gpt2} & $128\times128$ & 382 & 0,4,7,11\\
Qwen2.5 CP & \texttt{Qwen2.5-1.5B} & $128\times128$ & 384 & 9,18,27\\
ViT IN1K & \texttt{vit-base-patch16-224} & $197\times197$ & 512 & 0,4,7,11\\
ViT IN21K & \texttt{vit-base-patch16-224-in21k} & $197\times197$ & 512 & 0,4,7,11\\\bottomrule
\end{tabular}
\end{table}
The head roster is fixed at eight heads per family, with six retained for the common-prefix decoder panels after conditioning on non-local heads. GPT-2 contributes 382 rather than 384 cells because two fall below the retained-mass threshold; Qwen2.5 has no depth-zero stratum. The complete certificate run also includes 512 native-causal GPT-2 and 512 native-causal Qwen2 cells at 256-token support; these secondary cells contribute to the 9,978 bracket-integrity count but not to the geometry panel.

\subsection{Bracket integrity}
The final lower family was independently re-derived and recomputed on all 9,978 brackets. Targeted rank-$(L-1)$ attacks at the safe tolerance succeeded on 0/458 brackets, while the same machinery reaches the upper rank on closed controls. All 119 sampled upper constructions pass direct maximum-row-$\ell_1$ verification. These checks support bracket soundness; they do not identify the optimum when $L<U$.

At $\eps=0.25$, the median $U/L$ lies between 2.6 and 3.8 across family arms. BERT is the only family below the 50\% resolved-fraction component of the identification gate. The GPT-2 and Qwen2 ordering components depend on the decoded-inverse lower family; without it their ordering coverage drops substantially, although the resolved-fraction component remains high.

\subsection{ViT two-arm qualification}
The two ViT checkpoints share images, roster, operator shape, and code. Their upper-endpoint rankings are almost identical (paired Spearman $0.986$--$0.992$), yet IN1K upper endpoints exceed IN21K on roughly 82--85\% of cells, concentrated mostly in later layers. The arms are therefore reported separately throughout.

\subsection{Architecture-level common-shape audit}
A shape-matched $128\times128$ audit with 88 operators preserves the architecture contrast: median normalized constructive upper budgets $U/128$ are $0.14$--$0.17$ for encoder/vision operators versus $0.04$--$0.05$ for decoder blocks. This is descriptive architecture structure, not an optimum-rank claim.

\subsection{Primary falsification table}
\begin{table}[H]
\centering\small
\caption{Primary B1 versus B1+$G$ held-out results. Intervals are 95\% source-cluster bootstrap intervals. The aggregate equally weights the four primary families with ViT IN1K as the vision slot.}
\begin{tabular}{lrrrr}\toprule
Family & B1 MAE & B1+$G$ MAE & $\Delta$ & 95\% CI\\\midrule
BERT native & 0.0124 & 0.0128 & -0.00047 & $[-0.00107,+0.00005]$\\
GPT-2 CP & 0.0056 & 0.0056 & -0.00007 & $[-0.00047,+0.00029]$\\
Qwen2.5 CP & 0.0090 & 0.0093 & -0.00033 & $[-0.00088,+0.00036]$\\
ViT IN1K & 0.0247 & 0.0241 & +0.00068 & $[-0.00139,+0.00145]$\\
Aggregate & 0.0129 & 0.0130 & -0.00005 & $[-0.00060,+0.00025]$\\
ViT IN21K (disclosed) & 0.0203 & 0.0203 & +0.00002 & $[-0.00149,+0.00127]$\\\bottomrule
\end{tabular}
\end{table}
The aggregate relative gain is $\rho=-0.0035\;[-0.0422,+0.0179]$. At the upper confidence limit, $G$ removes at most about 2\% of the remaining B1 error in aggregate and about 6\% in the primary families. Because B1 MAE is already below the absolute resolution $\delta=0.02$ in BERT, GPT-2 CP, and Qwen2 CP, the absolute-scale null is not an equivalence test; the relative criterion supplies the informative resolution.

\subsection{Sensitivity and companion results}
The following analyses do not change the primary class:
\begin{itemize}[leftmargin=1.5em]
\item B2 (B1 plus attention spectra) versus B2+$G$: aggregate $\Delta\approx-0.00001$ with 95\% interval approximately $[-0.00034,+0.00014]$, relative interval approximately $[-0.038,+0.015]$.
\item A row-permuted placebo $G$ has gain about $-0.00023$, the same small order as the real block after B1.
\item Within-head demeaning gives $+0.00007\;[-0.00053,+0.00041]$; the native block does not track meaningful input-to-input variation within a head after B1.
\item Leave-one-head-out B1 versus B1+$G$ gives $+0.0001\;[-0.0023,+0.0049]$ and relative interval $[-0.031,+0.065]$. The weaker B0 cross-head comparison is heterogeneous and remains unresolved.
\item Attention-SVD constructive budget and domain-transfer targets remain null. Common-prefix geometry also adds no detectable information to native-causal GPT-2/Qwen2 targets, while attention spectra do add information there.
\item Against B0, $G$ gives the positive aggregate gain $+0.0279\;[+0.0254,+0.0300]$ with relative gain about $0.36$. A within-head permutation sensitivity retains most of this gain, which is consistent with coarse head-level structure; the leave-one-head-out B0 comparison remains heterogeneous and underpowered with only the fixed 6--8-head rosters.
\end{itemize}

\subsection{Gauge instability of the empirical native block}
Diagonal Q/K gauges with condition ratios 4 and 16 leave $A$ unchanged but move the frozen native predictors beyond the study-resolution comparison scale in every family. Thus the empirical block describes the trained factorization rather than an intrinsic operator property. This measurement is consistent with, but not needed for, the exact gauge theorem. It also does not test $\kap/q$: those intrinsic quantities were not extracted on the learned-head panel.

\section{Open questions and exact claim boundary}
The detailed theory leaves several explicit open problems:
\begin{enumerate}[leftmargin=1.6em]
\item whether the global numerical optimum of the intrinsic profile equals the optimum over all native gauges and orthogonal retained subspaces;
\item joint sharpness in intrinsic size, residual margin, and tolerance;
\item an efficient global algorithm for the rank-constrained profile $K_r(C;t)$;
\item empirical value of the new intrinsic $\kap/q$ descriptor on learned heads;
\item a theorem for full row-dependent causal masks.
\end{enumerate}
The paper does not infer runtime or compression speedup from approximation rank, does not assign exact optimum ranks where $L<U$, does not turn symbolic cover constants into finite predictions, and does not reinterpret the native-coordinate held-out null as a null about the intrinsic theorem.

\clearpage
\begin{thebibliography}{99}
\bibitem{altschuler2023kernel} Jason M. Altschuler and Pablo A. Parrilo. Kernel approximation on algebraic varieties. \emph{SIAM Journal on Applied Algebra and Geometry}, 7(1):1--28, 2023.
\bibitem{banerjee2005bregman} Arindam Banerjee, Srujana Merugu, Inderjit S. Dhillon, and Joydeep Ghosh. Clustering with Bregman divergences. \emph{JMLR}, 6:1705--1749, 2005.
\bibitem{bronshteyn1975polyhedra} E. M. Bronshteyn and L. D. Ivanov. The approximation of convex sets by polyhedra. \emph{Siberian Mathematical Journal}, 16:852--853, 1975.
\bibitem{budzinskiy2025bigdata} Stanislav Budzinskiy. When big data actually are low-rank, or entrywise approximation of certain function-generated matrices. \emph{SIAM Journal on Mathematics of Data Science}, 7(3):1098--1122, 2025.
\bibitem{carrell2025lowrank} Annabelle Michael Carrell, Albert Gong, Abhishek Shetty, Raaz Dwivedi, and Lester Mackey. Low-rank thinning. \emph{ICML}, 2025.
\bibitem{chen2021scatterbrain} Beidi Chen, Tri Dao, Eric Winsor, Zhao Song, Atri Rudra, and Christopher R\'e. Scatterbrain: Unifying sparse and low-rank attention approximation. \emph{NeurIPS}, 2021.
\bibitem{choromanski2021performer} Krzysztof Choromanski et al. Rethinking attention with Performers. \emph{ICLR}, 2021.
\bibitem{cohen2024hilbert} Samuel Cohen and Marco Fausti. Hyperbolic contractivity and the Hilbert metric on probability measures. arXiv:2309.02413v2, 2024.
\bibitem{gruber1993convex} Peter M. Gruber. Aspects of approximation of convex bodies. In \emph{Handbook of Convex Geometry}. North-Holland, 1993.
\bibitem{lee2026compressible} Wonsuk Lee. Compressible softmax-attended language under incompressible attention. arXiv:2604.04384, 2026.
\bibitem{lee2026invariants} Wonsuk Lee. On the invariants of softmax attention. arXiv:2605.02907, 2026.
\bibitem{liberty2026coresets} Edo Liberty, Alexandr Andoni, and Eldar Kleiner. Nearly optimal attention coresets. arXiv:2605.05602, 2026.
\bibitem{matousek2014factorization} Ji\v{r}\'i Matou\v{s}ek, Aleksandar Nikolov, and Kunal Talwar. Factorization norms and hereditary discrepancy. arXiv:1408.1376, 2014.
\bibitem{nishikawa2025dof} Naoki Nishikawa, Rei Higuchi, and Taiji Suzuki. Degrees of freedom for linear attention: Distilling softmax attention with optimal feature efficiency. \emph{NeurIPS}, 2025.
\bibitem{pan2024dissecting} Xu Pan, Aaron Philip, Ziqian Xie, and Odelia Schwartz. Dissecting query-key interaction in vision transformers. \emph{NeurIPS}, 2024.
\bibitem{wang2020linformer} Sinong Wang, Belinda Z. Li, Madian Khabsa, Han Fang, and Hao Ma. Linformer: Self-attention with linear complexity. arXiv:2006.04768, 2020.
\bibitem{xiong2021nystromformer} Yunyang Xiong et al. Nystr\"omformer: A Nystr\"om-based algorithm for approximating self-attention. \emph{AAAI}, 2021.
\end{thebibliography}
\end{document}